\documentclass[11pt]{article}

\usepackage{bbm}
\usepackage{eqnarray,amsmath,amsfonts,amsthm,mathrsfs}
\usepackage{color}
\usepackage{bm}
\usepackage{amssymb}
\usepackage[dvips]{graphicx}
\usepackage{epsfig}
\usepackage{epsf}
\usepackage{float}
\usepackage{subfigure}
\usepackage{graphicx,float,fancyhdr,multirow,hyperref}
\usepackage{booktabs,longtable,authblk}
\usepackage{hhline}
\usepackage[inline]{enumitem}
\usepackage[usenames]{xcolor}
\usepackage{inputenc}
\usepackage[T1]{fontenc}
\usepackage[mathcal]{eucal}    
\usepackage{hyperref}       
\usepackage{url}            
\usepackage{booktabs}           
\usepackage{nicefrac}       
\usepackage{microtype}      
\usepackage{lipsum}
\usepackage{mathtools}
\usepackage{cite}
\usepackage{nicematrix}
\usepackage{algorithm,algcompatible}
\usepackage[toc,page]{appendix}
\usepackage{lineno}
\usepackage{mathptmx}  
\usepackage{fancyhdr}
\usepackage[top=2.5cm, bottom=2.5cm, left=3cm, right=3cm]{geometry}
\usepackage{graphicx}
\usepackage{appendix}
\newtheorem{theorem}{Theorem}

\newtheorem{corollary}{Corollary}

\newtheorem{lemma}{Lemma}
\newtheorem{proposition}{Proposition}
\newtheorem{remark}{Remark}

\numberwithin{equation}{section}
\DeclareMathOperator{\vect}{vec}
\DeclareMathOperator{\size}{size}
\DeclareMathOperator{\sq}{sq}
\DeclareMathOperator{\prd}{prd}
\DeclareMathOperator*{\esssup}{ess\,sup}

\title{Two-Dimensional Deep ReLU CNN Approximation for Korobov Functions: A Constructive Approach$^\dag$\footnotetext{\dag~The work of Lei Shi is supported by National Natural Science Foundation of China [Project No.12571099]. The work of Min Xu is supported by National Nature Science Foundation of China [Project No. 62273070]. The work of Ding-Xuan Zhou is partially supported by the Australian Research Council under project DP240101919. Authors are listed in alphabetical order and contributed equally to this work. Email addresses: fangqin@dlu.edu.cn (Q. Fang), leishi@fudan.edu.cn (L. Shi), wolf\_hsu@dlut.edu.cn (M. Xu), dingxuan.zhou@sydney.edu.au (D.-X. Zhou).}}

\author{Qin Fang}
\affil[1]{Information and Engineering College, Dalian University, Dalian 116622, China}
\author{Lei Shi}
\affil[2]{School of Mathematical Sciences and Shanghai Key Laboratory for Contemporary Applied Mathematics, Fudan University, Shanghai 200433, China}
\author{Min Xu}
\affil[3]{School of Mathematical Sciences, Dalian University of Technology, Dalian 116024, China}
\author{Ding-Xuan Zhou}
\affil[4]{School of Mathematics and Statistics, University of Sydney, Sydney NSW 2006, Australia}

\date{}

\begin{document}
	
\maketitle

\begin{abstract}
This paper investigates approximation capabilities of two-dimensional (2D) deep convolutional neural networks (CNNs), with Korobov functions serving as a benchmark. We focus on 2D CNNs, comprising multi-channel convolutional layers with zero-padding and ReLU activations, followed by a fully connected layer. We propose a fully constructive approach for building 2D CNNs to approximate Korobov functions and provide rigorous analysis of the complexity of the constructed networks. Our results demonstrate that 2D CNNs achieve near-optimal approximation rates under the continuous weight selection model, significantly alleviating the curse of dimensionality. This work provides a solid theoretical foundation for 2D CNNs and illustrates their potential for broader applications in function approximation.
\end{abstract}


\section{Introduction}

Deep learning techniques based on deep neural networks have achieved a remarkable success across various domains, including computer vision, natural language processing, and speech recognition \cite{goodfellow2016deep, lecun2015deep}. Among these, convolutional neural networks (CNNs) have become a fundamental model, demonstrating exceptional performances in tasks such as object detection \cite{zhao2019object}, image classification \cite{li2014medical}, and scientific computing \cite{lei2025solving}. From early architectures such as AlexNet \cite{krizhevsky2012imagenet} and VGGNet \cite{simonyan2015very} to more advanced models like MgNet \cite{he2019mgnet} and ResNet \cite{he2016deep}, CNNs have consistently outperformed fully connected neural networks (FNNs) in vision-related tasks \cite{he2016deep, huang2017densely}. Despite their empirical success, mathematical foundations of CNNs remain underdeveloped. Particularly, their approximation capabilities, generalization properties, and optimization dynamics are still open research questions critical to advancing CNN-based methods \cite{han2023deep, kohler2020statistical, zhou2024learning}. This paper focuses on addressing one of these gaps by studying the approximation capabilities of two-dimensional (2D) CNNs, with particular attention to their ability to approximate functions from Korobov spaces.

Research on the theoretical foundations of CNNs, including their approximation properties and inductive biases, remains relatively limited. Wang et al. \cite{wang2023theoretical} provided a theoretical analysis of the inductive biases of CNN architectures, demonstrating how architectural constraints such as local connectivity, weight sharing, and depth influence the function classes representable by CNNs. Their analysis offers valuable insights into the structural advantages of CNNs beyond empirical observations. Zhou \cite{zhou2020universality} made a pivotal contribution by establishing the universality of classical one-dimensional (1D) CNNs. Using a decomposition theorem for large 1D convolutional kernels, Zhou \cite{zhou2020theory} demonstrated that any ReLU FNN can be equivalently expressed as a 1D ReLU CNN. This result extends the approximation theory for FNNs to 1D CNNs, significantly advancing our theoretical understanding of CNNs. Based on these works, Mao and Zhou \cite{mao2022approximation} introduced a constructive framework for analyzing the approximation capabilities of 1D CNNs. Inspired by techniques developed for deep ReLU FNNs \cite{yarotsky2017error}, they constructed a product network through convolutions, which forms the core of their approximation methodology. This product network is then used to construct CNNs that can approximate hierarchical basis functions well, enabling the effective approximation of Korobov functions. However, these studies \cite{mao2023approximating} focus primarily on 1D data, such as audio signals, limiting their applicability to higher-dimensional data like images or videos.

Exploring the approximation capabilities of 2D CNNs presents additional challenges, as it involves capturing interactions across both horizontal and vertical dimensions. He et al. \cite{he2022approximation} addressed this by developing a novel decomposition theorem for 2D convolutional kernels with large spatial sizes and multi-channels. They showed that any shallow ReLU neural network (NN) on the tensor space $[0,1]^{d\times d}$ can be equivalently expressed as a 2D ReLU CNN, consisting of multi-channel convolutional layers with zero-padding, ReLU activation, and a fully connected layer. This result facilitates the extension of approximation theory from shallow NNs to 2D CNNs, providing valuable insights into 2D CNNs.

While the equivalence between FNNs and periodized 2D CNNs has been studied \cite{petersen2020equivalence}, with the result that any FNN can be expressed as a CNN in a specific, non-standard architecture, this work has several limitations. One key issue is the use of periodic padding in the convolution operations—a technique relatively uncommon in modern CNN workflows. Periodic padding creates artificial continuity at data boundaries, which can distort edge features, making it unsuitable for tasks such as image recognition and segmentation, where preserving edge information is critical \cite{ronneberger2015u}. Additionally, the fixed kernel size of $d\times d$ in the architecture, which matches the dimensions of the input matrix, lacks the flexibility needed to capture varying receptive fields—an essential feature of modern CNNs. In contrast, widely used architectures like VGGNet \cite{simonyan2015very} and ResNet \cite{he2016deep} utilize smaller, trainable kernels (e.g., $3\times 3$ or $5\times 5$), which enable more efficient computation and better hierarchical feature extraction. Furthermore, the rigid handling of multi-channel inputs in the architecture does not exploit the hierarchical feature learning capabilities that are central to contemporary CNNs, which progressively capture both low- and high-level features for superior performances in a variety of tasks \cite{karniadakis2021physics,lecun2015deep}. These limitations motivate the need to analyze CNNs more commonly used in practice, such as the one studied in this paper.

The main contributions of this paper are as follows:
\begin{enumerate}
    \item \textnormal{\textbf{Analysis of 2D ReLU CNNs.}}      
    This paper analyzes the approximation capabilities of 2D CNNs consisting of a sequence of multi-channel convolutional layers followed by a fully connected layer. This network architecture utilizes zero-padding, ReLU activations, and smaller, trainable kernels (e.g., $3 \times 3$ or $5 \times 5$), which are widely adopted in modern deep learning frameworks such as PyTorch and TensorFlow.
       
    \item \textnormal{\textbf{Approximation of Korobov functions.}}  This paper investigates the approximation of functions from Korobov spaces, a topic with both theoretical and practical importance. Korobov spaces play a crucial role in high-dimensional approximation, offering a framework to alleviate the curse of dimensionality. They are extensively used in applications such as numerical partial differential equations (PDEs) and high-dimensional function approximation. Through constructive analysis of how 2D CNNs approximate Korobov functions, this work provides insights that can enhance numerical methods in scientific computing.

    \item \textnormal{\textbf{Constructiveness and Optimality.}}
    The proofs of our main approximation results are fully constructive, providing explicit, layer-by-layer specification of 2D CNNs that achieve the stated approximation rates. Unlike the approach in \cite{mao2022approximation}, which proves the existence of suitable approximations via a decomposition theorem, we provide the explicit construction itself in Theorem \ref{thm:approrate} and Corollary \ref{cor:apprate}. Moreover, we show that this construction is near-optimal under a continuous weight model, with the network size scaling nearly optimally with respect to the approximation error.
\end{enumerate}

While our constructive approach is inspired by the 1D framework of \cite{mao2022approximation}, extending it to two dimensions introduces fundamental spatial–combinatorial challenges that constitute the core contribution of this work. In the 1D case, the key operation underlying the construction of hierarchical basis functions—computing the global product of all elements in a vector—can be realized via sequential shifts along a single axis, implemented through simple convolutional filters. In the 2D setting, however, the analogous operation requires computing the global product of all entries in a 2D tensor. This necessitates coordinating shifts across a plane rather than along a line, involving interactions in multiple directions (horizontal, vertical, and diagonal). Importantly, this change leads to a qualitatively different convolutional algebra, in which product-type interactions can no longer be organized through a linear ordering of components. As a result, the 2D construction must explicitly encode multi-directional spatial coordination using only small kernel convolutions, zero-padding, and ReLU activations—the standard building blocks of practical CNN architectures. Propositions \ref{pro:app4prod4tensor} and \ref{pro:app4prod4basis} address this challenge by introducing basic kernel blocks $S^{s,t}$
and a corresponding network composition, thereby establishing a fully constructive approximation theory tailored specifically to 2D CNNs.

We emphasize that the main approximation results in this paper are established with respect to the $L^p$ norm. Extending such results to stronger norms that involve derivatives, such as Sobolev or Korobov-type norms, would require accurate control of derivatives of the CNN approximants. Particularly, this entails stable neural network realizations of product-type constructions in Sobolev-type norms, which are technically more challenging. While recent works have established approximation results for FNNs in Sobolev norms using more involved constructions and estimates \cite{yang2024near, zhou2025expressive}, developing analogous results for 2D CNNs remains an open and nontrivial problem. We thus restrict our attention to $L^p$-approximation and leave extensions to stronger norms for future research.

This paper is structured as follows. In Section \ref{sec:prelim}, we provide the necessary preliminaries, including essential notations, an overview of 2D deep ReLU CNNs, and an introduction to Korobov spaces and sparse grids. Section \ref{sec:main} presents our main results, including a theorem and its corollary, which address the approximation rate and network complexity for 2D deep ReLU CNNs in the context of Korobov spaces. Section \ref{sec:proof4thm1} contains the proof of the main results. We present two propositions: one concerning the product of all elements in a tensor and the other related to the approximation of hierarchical basis functions, followed by a detailed proof of our main theorem. In Section \ref{sec:conclu}, we conclude our findings and directions for future research. Finally, the appendix includes basic CNN constructions, detailed proofs of the propositions, and two technical lemmas.

\section{Preliminaries}\label{sec:prelim}
\subsection{Notations}
Let $\mathbb{R}$ represent the set of real numbers, $\mathbb{Z}$ stand for the set of integers, $\mathbb{Z}_+$ denote the set of non-negative integers, and $\mathbb{N}$ signify the set of positive integers. For $c^\prime,c,d^\prime,d\in\mathbb{N}$, the notation $\mathbb{R}^{c^\prime\times c\times d^\prime\times d}$ denotes the set of four-dimensional tensors with real-numbered elements. In this notation, the dimensions along its four axes are $c^\prime$, $c$, $d^\prime$ and $d$. Furthermore, the notation $\mathbb{R}^{c^\prime\times c\times \mathbb{N}\times \mathbb{N}}$ represents the set of four-dimensional tensors, where the first and second dimensions are fixed with sizes $c^\prime$ and $c$, respectively, meanwhile, the third and fourth dimensions are permitted to vary over the integer set $\mathbb{N}$. For $a\in\mathbb{R}$, we use $\bm{a}_{c^\prime\times c\times d^\prime\times d}$ to denote the tensor in $\mathbb{R}^{c^\prime\times c\times d^\prime\times d}$ with all elements equal to $a$. Similar notations apply to $\bm{a}_{d^\prime\times d}$ and $\bm{a}_{c^\prime}$. Let $\lfloor\cdot\rfloor$ denote the floor function, which rounds down to the nearest integer, and $\lceil\cdot\rceil$ denote the ceiling function, which rounds up to the nearest integer. We use $\mathcal{O}$ to indicate an upper bound on the asymptotic growth of a function.

\subsection{2D ReLU CNNs}\label{subsec:2drelucnn}
Let us introduce some fundamental mathematical concepts used in 2D deep ReLU CNNs.

\noindent\textbf{Data tensor.} A data tensor, denoted as $X$, has $c$ channels and spatial dimensions $d\times d$. It is represented as $X\in \mathbb{R}^{c\times d\times d}$, with individual elements $[X]_{q,m,n}$ indexed by $q\in 1:c$ and $m,n\in 1:d$, where the notation $s:t$ signifies the set $\{s,s+1,\ldots,t\}$. Let $X_q=[X]_{q,:,:}\in\mathbb{R}^{d\times d}$ denote the matrix corresponds to the $q$-th channel of $X$. Then the entire data tensor $X$ can be formally expressed as
\begin{equation*}
X=\begin{pmatrix}
X_1\\
\vdots\\
X_c
\end{pmatrix}.
\end{equation*}

\noindent\textbf{Zero padding.} The zero padding operation on the tensor space $\mathbb{R}^{c\times d\times d}$ is expressed as a mapping $\iota:\mathbb{R}^{c\times d\times d}\rightarrow\mathbb{R}^{c\times \mathbb{Z}\times \mathbb{Z}}$ defined by
\begin{equation*}
	[\iota(X)]_{q, m, n}:=
	\begin{cases}
		[X]_{q, m, n}, & \text{if}~m,n\in 1:d, \\
		0, & \text{otherwise},
	\end{cases}
\end{equation*}
for all channels $q\in 1:c$. According to this definition, if the spatial coordinates $m,n\in\mathbb{Z}$ fall within $1:d$, the corresponding element remains unchanged. However, the element is padded with zero if either $m$ or $n$ extends beyond the range, indicating a need for additional spatial context.

\noindent\textbf{Convolution kernel tensor.} The convolution kernel, denoted as $K$, is characterized by $c$ input channels and $c^\prime$ output channels, and possesses a spatial size of $2k+1$. Represented as
\begin{equation*}
	K\in \mathbb{R}^{c^\prime\times c\times (2k+1)\times(2k+1)},
\end{equation*}
its individual elements $[K]_{p,q,s,t}$ are then indexed by $p\in 1:c^\prime$, $q\in 1:c$, and $s,t\in -k:k$. Let $K_{p,q}=[K]_{p,q,:,:}\in\mathbb{R}^{(2k+1)\times (2k+1)}$ denote the matrix corresponding to the $p$-th output channel and the $q$-th input channel of $K$. Then the kernel tensor $K$ can be formally expressed as
\begin{equation*}
K=\begin{pmatrix}
K_{1,1}        & \cdots &K_{1,c}\\
\vdots         & \ddots &\vdots\\
K_{c^\prime,1} & \cdots &K_{c^\prime,c}
\end{pmatrix}.
\end{equation*}

\noindent\textbf{Zero-padding convolution.} The multi-channel convolution with a kernel $K\in \mathbb{R}^{c^\prime\times c\times (2k+1)\times (2k+1)}$ is expressed as a mapping $A_K:\mathbb{R}^{c\times d \times d}\rightarrow\mathbb{R}^{c^\prime\times d \times d},~~X\mapsto K\ast X$, where $K\ast X$ is given by the following equations
\begin{equation*}
	[K \ast X]_{p,m,n}=\sum_{q=1}^{c} \sum_{s,t=-k}^{k} [K]_{p,q,s,t} [\iota(X)]_{q,m + s, n+t},
\end{equation*}
for all $p\in 1:c^\prime$ and $m,n\in 1:d$. This equation incorporates zero-padding to address cases where $m+s$ or $n+t$ exceed the range $1:d$. Employing the established notations for $K$ and $X$, the convolution $K\ast X$ can be alternatively expressed as
\begin{equation*}
	[K\ast X]_p=\sum_{q=1}^{c} K_{p,q}\ast X_q,~~~\text{for}~p\in 1:c^\prime,
\end{equation*}
where $K_{p,q} \ast X_q\in \mathbb{R}^{d \times d}$ denotes the single-channel convolution, i.e.,
\begin{equation*}
	[K_{p,q} \ast X_q]_{m,n}=\sum_{s,t=-k}^{k} [K_{p,q}]_{s,t} [\iota(X_q)]_{m+s,n+t},~~~\text{for}~m,n\in 1:d.
\end{equation*}
It is important to note that the convolution operation defined above does not satisfy the commutative or associative laws. The default interpretation of $K^2\ast K^1\ast X$ is given by $K^2\ast (K^1\ast X)$.

\noindent\textbf{ReLU activation function.} Let $\sigma: \mathbb{R}\rightarrow\mathbb{R}$ denote the Rectified Linear Unit (ReLU) activation function, defined as follows:
\begin{equation*}
	\sigma(x):=
	\begin{cases}
		x, & \text{if}~x\geq 0, \\
		0, & \text{otherwise}.
	\end{cases}
\end{equation*}
With a slight abuse of notation, we use $\sigma(X)$ to denote the application of the function $\sigma$ to each component of $X\in\mathbb{R}^{c\times d \times d}$ separately. This operation is formally expressed as
\begin{equation*}
	[\sigma(X)]_{q,m,n}=\sigma([X]_{q,m,n}),
\end{equation*}
for all $q\in 1:c$ and $m,n\in 1:d$. 

\noindent\textbf{2D deep ReLU CNNs.} An $L$-layered ReLU CNN, with channel size vector $\bm{c}=(c_0,\ldots,c_L)\in\mathbb{N}^{L+1}$ and kernel spatial size vector $\bm{s}=(2k_1+1,\ldots,2k_L+1)\in\mathbb{N}^L$, is a mapping $\bm{h}^L:\mathbb{R}^{c_0\times d\times d}\rightarrow\mathbb{R}^{c_L\times d\times d}$, defined iteratively as
\begin{equation*}
\bm{h}^{l}(X):=\sigma(K^l\ast \bm{h}^{l-1}(X)+\bm{b}^l\bm{1}_{d\times d}),~~~\text{for}~l\in 1:L,
\end{equation*}
$\bm{h}^0(X):=X\in\mathbb{R}^{c_0\times d\times d}$ is the input tensor, $K^l\in\mathbb{R}^{c_{l} \times c_{l-1}\times (2k_l+1) \times (2k_l+1)}$ are convolution kernels, $\bm{b}^l\in\mathbb{R}^{c_l}$ are biases, and $\bm{1}_{d\times d}\in \mathbb{R}^{d\times d}$ is the matrix with all elements equal to $1$. The term $\bm{b}^l\bm{1}_{d\times d}$ is defined as the following tensor
\begin{equation*}
    \bm{b}^l\bm{1}_{d\times d}:=\begin{pmatrix}
    [\bm{b}^l]_1\bm{1}_{d\times d}\\
    \vdots\\
    [\bm{b}^l]_{c_l}\bm{1}_{d\times d}
    \end{pmatrix}\in \mathbb{R}^{c_l\times d\times d}.
\end{equation*}
Let $A_{K,\bm{b}}$ denote the mapping 
\begin{equation*}
	\mathbb{R}^{c\times d\times d}\rightarrow\mathbb{R}^{c^\prime\times d\times d},~~~Y\mapsto K\ast Y + \bm{b}\bm{1}_{d\times d},
\end{equation*}
where $K\in\mathbb{R}^{c^\prime\times c\times (2k+1)\times (2k+1)}$ and $\bm{b}\in\mathbb{R}^{c^\prime}$. Then $\bm{h}^L$ can be expressed as the following compositions
\begin{equation*}
    \bm{h}^L=\sigma\circ A_{K^L,\bm{b}^L}\circ\cdots\circ\sigma\circ A_{K^1,\bm{b}^1}.
\end{equation*}

The set of mappings $\bm{h}^L$ generated by the CNN architecture, specified by $L$, $\bm{c}$, and $\bm{s}$, and considering all possible convolution kernels $K^l$ and biases $\bm{b}^l$, is denoted as
\begin{equation*}
\mathcal{C}^{\bm{c},L}_{\bm{s}}(\mathbb{R}^{c_0\times d\times d},\mathbb{R}^{c_L\times d\times d}).
\end{equation*}
The size of $\bm{h}^L$, denoted by $\size(\bm{h}^L)$, is defined as the total number of possibly nonzero elements in the kernels $K^l$ and biases $\bm{b}^l$. Let $\vect(\bm{h}^L(X))\in \mathbb{R}^{c_Ld^2}$ denote the vectorization of  $\bm{h}^L(X)$, which is defined as
\begin{equation*}\label{key}
[\vect(\bm{h}^L(X))]_{(q-1)d^2+(m-1)d+n}=[\bm{h}^L(X)]_{q,m,n},
\end{equation*} 
for $q\in 1:c_L$ and $m,n\in 1:d$. The hypothesis space $\mathcal{H}^{\bm{c},L}_{\bm{s}}(\mathbb{R}^{c_0\times d\times d})$ for the network architecture is the span of the constant $1$ function and the functions $[\vect(\bm{h}^L(X))]_i$ for all $\bm{h}^L\in \mathcal{C}^{\bm{c},L}_{\bm{s}}(\mathbb{R}^{c_0\times d\times d},\mathbb{R}^{c_L\times d\times d})$ and $i\in 1:c_Ld^2$, i.e.,
\begin{equation*}
\begin{aligned}
\mathcal{H}^{\bm{c},L}_{\bm{s}}(\mathbb{R}^{c_0\times d\times d}) 
:= \bigg\{ \beta + \sum_{i=1}^{c_Ld^2} \alpha_i 
[\vect(\bm{h}^L(X))]_i : 
& \ \beta \in \mathbb{R}, \ \bm{\alpha} \in \mathbb{R}^{c_Ld^2}, \\
& \ \bm{h}^L \in \mathcal{C}^{\bm{c},L}_{\bm{s}} 
(\mathbb{R}^{c_0\times d\times d},\mathbb{R}^{c_L\times d\times d}) 
\bigg\}.
\end{aligned}
\end{equation*}
The size of $h\in\mathcal{H}^{\bm{c},L}_{\bm{s}}(\mathbb{R}^{c_0\times d\times d})$, denoted by $\size(h)$, is defined as the total number of possibly nonzero elements in the corresponding kernels $K^l$, biases $\bm{b}^l$, and coefficients $\beta,\bm{\alpha}$.

The expression ``a CNN (architecture) with width $W$, depth $L$, and kernel spatial size $2k+1$'' means that: (a) the maximum channel size in hidden layers of the network (architecture) is $W$, i.e., $W=\max\{c_1,\ldots,c_L\}$; (b) the network (architecture) consists of $L$ layers; and (c) the spatial size of the kernels in each layer is consistently $2k+1$. We use the notation $\mathcal{C}^{W,L}_{2k+1}(\mathbb{R}^{c_0\times d\times d},\mathbb{R}^{c_L \times d\times d})$ to denote the set of mappings $\bm{h}^L$ produced by the CNN architecture with width $W$, depth $L$, and kernel spatial size $2k+1$, and $\mathcal{H}^{W,L}_{2k+1}(\mathbb{R}^{c_0\times d\times d})$ to represent the corresponding hypothesis space.

\subsection{Korobov Spaces}\label{subsec:korov}
Let $\Omega=[0,1]^D$ for some $D\in\mathbb{N}$, and let $1\leq p\leq\infty$. The Lebesgue space $L^p(\Omega)$ consists of measurable functions $f$ on $\Omega$ such that the norm
\begin{equation*}
\|f\|_{L^p(\Omega)}:=
\begin{cases}
\Big({\displaystyle\int_{\Omega}}|f(\bm{x})|^p {\rm d}\bm{x}\Big)^{\frac{1}{p}}, & ~~\text{if}~1\leq p<\infty, \\
\esssup\limits_{\bm{x}\in \Omega} |f(\bm{x})|, & ~~\text{if}~p=\infty
\end{cases}
\end{equation*}
is finite. For $r\in\mathbb{N}$, the Korobov space $X^{r,p}(\Omega)$ is defined as the space of functions $f\in L^p(\Omega)$ that vanish on the boundary of $\Omega$ and whose weak mixed partial derivatives up to order $r$ belong to $L^p(\Omega)$
\begin{equation*}
X^{r,p}(\Omega):=\big\{f\in L^p(\Omega):f|_{\partial\Omega}=0,\partial^{\bm{\alpha}}f\in L^p(\Omega)~\text{for}~|\bm{\alpha}|_{\infty}\leq r\big\}.
\end{equation*}
The norm on $X^{r,p}(\Omega)$ is defined as
\begin{equation*}
\|f\|_{X^{r,p}(\Omega)}:=
\begin{cases}
\bigg(\sum\limits_{|\bm{\alpha}|_{\infty}\leq r}\|\partial^{\bm{\alpha}}f\|_{L^p(\Omega)}^p\bigg)^{\frac{1}{p}}, & ~~\text{if}~1\leq p<\infty, \\
\max\limits_{|\bm{\alpha}|_{\infty}\leq r}\|\partial^{\bm{\alpha}}f\|_{L^{\infty}(\Omega)}, &~~\text{if}~p=\infty.
\end{cases}
\end{equation*}

Korobov spaces are fundamental for high-dimensional approximation, providing a framework to alleviate the curse of dimensionality \cite{novak2008tractability}. They are widely used in areas such as numerical PDEs \cite{bungartz2004sparse} and high-dimensional function approximation \cite{dung2018hyperbolic}. In this paper, we leverage sparse grid methods for Korobov spaces $X^{2,p}(\Omega)$, which are essential for the construction of our deep neural networks. For a comprehensive overview of sparse grids and their applications, we refer the reader to \cite{bungartz2004sparse}. 

The fundamental component of sparse grids is a basis of high-dimensional functions, which is constructed by multiplying 1D hat functions. Specifically, consider the 1D hat function $\phi:\mathbb{R}\rightarrow\mathbb{R}$ defined by
\begin{equation*}
\phi(x):= 
\begin{cases}
1-|x|, & \text {if}~x\in[-1,1], \\ 
0,     & \text {otherwise}.
\end{cases}
\end{equation*}
For any level $l\in\mathbb{N}$, define the grid size as $h_l:=2^{-l}$, and the corresponding grid points on the interval $[0,1]$ as $x_{l,i}:=ih_l$, where $i\in\mathbb{N}$ and $1\leq i\leq 2^l-1$. Using these grid points, we define a family of 1D hat functions $\phi_{l, i}:\mathbb{R}\rightarrow\mathbb{R}$ by
\begin{equation*}
\phi_{l, i}(x):=\phi\left(\frac{x-x_{l,i}}{h_l}\right),\quad\text{for}~x\in\mathbb{R}.
\end{equation*}

We construct a basis for the space $X^{2,p}(\Omega)$. To illustrate, for any $\bm{l}\in\mathbb{N}^{D}$ and $\bm{i}\in\mathbb{N}^{D}$ with $\bm{1}_{D}\leq\bm{i}\leq 2^{\bm{l}}-\bm{1}_{D}$ (where the exponential and inequalities are understood component-wise), consider the function $\phi_{\bm{l},\bm{i}}$ defined by the product of the 1D hat functions,
\begin{equation*}
\phi_{\bm{l},\bm{i}}(\bm{x}):=\prod_{j=1}^D \phi_{l_j, i_j}\left(x_j\right),\quad \bm{x}=(x_1,\ldots,x_D)^T\in \mathbb{R}^{D}.
\end{equation*}
According to \cite[Lemma 3.1]{bungartz2004sparse}, the function $\phi_{\bm{l},\bm{i}}$ satisfies
\begin{equation}\label{eq:normofbasis}
\|\phi_{\bm{l},\bm{i}}\|_{L^p(\Omega)}= 
\begin{cases}
\left(\frac{2}{p+1}\right)^{\frac{D}{p}}\cdot2^{-\frac{\|\bm{l}\|_1}{p}}, &~~\text{if}~1\leq p<\infty, \\ 
1,     &~~\text{if}~p=\infty.
\end{cases}
\end{equation}
Moreover, it has been established \cite{bungartz2004sparse} that any function $f\in X^{2,p}(\Omega)$, where $2\leq p\leq\infty$, admits a unique expansion in the hierarchical basis $\{\phi_{\bm{l},\bm{i}}(\bm{x}):\bm{i}\in I_{\bm{l}},\bm{l}\in\mathbb{N}^D\}$,
\begin{equation}\label{eq:rep4kov}
f(\bm{x})=\sum_{\bm{l}\in \mathbb{N}^D}\sum_{\bm{i}\in I_{\bm{l}}}v_{\bm{l},\bm{i}}\phi_{\bm{l},\bm{i}}(\bm{x}),
\end{equation}
where $I_{\bm{l}}$ denotes the index set 
\begin{equation*}
I_{\bm{l}}:=\{\bm{i}\in\mathbb{N}^{D}:\bm{1}_{D}\leq \bm{i}\leq 2^{\bm{l}}-\bm{1}_{D},~i_j~\text{is odd for}~1\leq j\leq D\}.
\end{equation*}
The coefficients $v_{\bm{l},\bm{i}}\in\mathbb{R}$ are defined by
\begin{equation*}
v_{\bm{l},\bm{i}}=-2^{-|\bm{l}|_1-D}\int_{\Omega}\phi_{\bm{l},\bm{i}}(\bm{x})\,\partial^{\bm{2}}f(\bm{x})\,d\bm{x}
\end{equation*} 
and satisfy the estimate
\begin{equation}\label{eq:bou4cof}
|v_{\bm{l},\bm{i}}| \leq  2^{-|\bm{l}|_1-D}\|\phi_{\bm{l},\bm{i}}\|_{L^q(\Omega)}\,\|f|_{\operatorname{supp}(\phi_{\bm{l},\bm{i}})}\|_{X^{2,p}(\Omega)},
\end{equation}
where $q$ denotes the conjugate exponent of $p$, and $\operatorname{supp}(\phi_{\bm{l},\bm{i}})$ denotes the support of $\phi_{\bm{l},\bm{i}}$. Since the sum in (\ref{eq:rep4kov}) is infinite, an important challenge is determining how to truncate it to achieve an approximation of $f$. For any $n\in\mathbb{N}$, sparse grids provide the following truncated approximation of $f$,
\begin{equation}\label{eq:sparseapp}
f_n^{(1)}(\bm{x}):=\sum_{|\bm{l}|_1\leq n+D-1}\sum_{\bm{i}\in I_{\bm{l}}}v_{\bm{l},\bm{i}}\phi_{\bm{l},\bm{i}}(\bm{x}),
\end{equation}
for which the approximation error satisfies
\begin{equation}\label{eq:sparseappnonconst}
\big\|f-f_n^{(1)}\big\|_{L^p(\Omega)}=O(2^{-2n}n^{D-1}),\qquad 2\le p\le\infty.
\end{equation}
The detailed proof of this estimate, together with explicit constants, will be provided in the proof of Theorem \ref{thm:approrate}.

\section{Main Results}\label{sec:main}

In this work, we study the approximation of functions in the space $X^{r,p}(\Omega)$ and focus on the theoretically critical case $r=2$. This choice is dictated by the use of piecewise linear sparse grid approximations, which already attain the optimal convergence order for $r=2$ and form the basis of our fully constructive CNN realization.

\begin{theorem}\label{thm:approrate}
	Let $k,d\in\mathbb{N}$ with $d\geq 3$, and let $\Omega=[0,1]^{d\times d}$. Suppose that $f\in X^{2,p}(\Omega)$ with $2\leq p\leq \infty$ satisfies $\|f\|_{X^{2,p}(\Omega)}\leq 1$. Then, for sufficiently large $N\in\mathbb{N}$ (as detailed in the proof), there exists a CNN $h\in\mathcal{H}^{W,L}_{2k+1}(\mathbb{R}^{d\times d})$ with width $W=2Nd^2$ and depth $L=2(2\lceil\log_2N\rceil+3)\lceil\log_2d\rceil+6d$ such that 
	\begin{equation}\label{eq:mainerror}
	\big\|f-h\big\|_{L^p(\Omega)}\leq  \frac{4}{2^{d^2}}\frac{\big(\log_2N\big)^{3(d^2-1)}}{N^2}. 
	\end{equation}
	Moreover, the size of $h$ satisfies
	\begin{equation}\label{eq:mainnumpar}
		\size(h)
		\leq 24(2k+1)^2d^5N\log_2N.
	\end{equation}
\end{theorem}

\begin{remark}[Saturation with respect to smoothness]
The focus on $r=2$ reflects the use of piecewise linear sparse grid approximations and the associated saturation phenomenon. For Korobov functions, once the mixed smoothness exceeds $r=2$, the convergence order of the sparse grid truncation $f_n^{(1)}$ based on piecewise linear hierarchical basis functions no longer improves \cite{bungartz2004sparse}. This limitation is intrinsic to the approximation scheme and not to the CNN realization itself. Exploiting higher smoothness would require sparse grids built from higher-order piecewise polynomial bases, which in turn would necessitate substantially more involved and technically demanding CNN constructions.
\end{remark}

\begin{remark}[Influence of the kernel size]
The approximation rate established in Theorem \ref{thm:approrate} holds for all kernel sizes $2k+1$ and is independent of $k$. Particularly, the minimal kernel size ($k=1$) already suffices to achieve the stated convergence rate for functions from the Korobov space $X^{2,p}(\Omega)$. This observation is consistent with existing results: To the best of our knowledge, there is currently no theoretical guarantee that links the kernel size to the convergence rate of CNN approximation. Nevertheless, the kernel size has a significant impact on network complexity. A more delicate modification of the construction shows that larger kernels enable information to be aggregated over a broader spatial neighborhood within a single convolutional layer, which in turn reduces the required network depth. More specifically, when $(2k+1)\times(2k+1)$ kernels are employed, the resulting network depth satisfies $L=\mathcal{O}\big(\log_2N\,\log_2d+d/k\big)$, while the network size is bounded as $\size(h)=\mathcal{O}\big(k\,d^5N\log_2N\big)$. This result makes explicit the trade-off between kernel size and network depth, while leaving the approximation rate unchanged.
\end{remark}

To guarantee an accuracy $\epsilon>0$, it suffices to choose $N$ such that
\begin{equation*}
\frac{4}{2^{d^2}}\frac{\big(\log_2N\big)^{3(d^2-1)}}{N^2}\leq\epsilon.
\end{equation*}
One convenient choice is
\begin{equation*}
N=\Big\lceil (9\beta\log_2\beta)^{\beta}\gamma^{-1}\epsilon^{-\frac{1}{2}}|\log_2\epsilon|^{\beta}\Big\rceil,
\end{equation*}
where $\beta=3(d^2-1)/2$ and  $\gamma=2^{d^2/2-1}$. Indeed, by Lemma \ref{lem:boundcompl} in the appendix, for sufficiently small $\epsilon>0$, the inequality
\begin{equation*}
\frac{\log_2^{\beta}N}{N}\leq \gamma \epsilon^{\frac{1}{2}}
\end{equation*}
holds, which ensures the desired error bound. As a consequence, we get the following corollary.

\begin{corollary}\label{cor:apprate}
	Let $k,d\in\mathbb{N}$ with $d\geq 3$, and let $\Omega=[0,1]^{d\times d}$. Suppose that $f\in X^{2,p}(\Omega)$ with $2\leq p\leq \infty$ satisfies $\|f\|_{X^{2,p}(\Omega)}\leq 1$. Then, for sufficiently small $\epsilon>0$, there exists a CNN $h\in\mathcal{H}^{W,L}_{2k+1}(\mathbb{R}^{d\times d})$ with width $W=\mathcal{O}(\epsilon^{-1/2}|\log_2\epsilon|^{3(d^2-1)/2})$ and depth $L=\mathcal{O}(|\log_2\epsilon|)$, such that 
	\begin{equation*}
	\big\|f-h\big\|_{L^p(\Omega)}\leq\epsilon.
	\end{equation*}
	Moreover, the size of $h$ satisfies
	\begin{equation*}
	\size(h)=\mathcal{O}\big(\epsilon^{-\frac{1}{2}}|\log_2\epsilon|^{\frac{3(d^2-1)}{2}+1}\big).
	\end{equation*}
\end{corollary}

Before turning to the proof of the theorem, we briefly compare our results with existing approximation methods for Korobov functions.

Montanelli and Du \cite{montanelli2019new} studied the approximation of Korobov functions using deep ReLU FNNs. With the error measured in the $L^\infty([0,1]^d)$ norm, they showed that an accuracy $\epsilon$ can be achieved with network size of order
\begin{equation*}
\mathcal{O}\big(\epsilon^{-\frac{1}{2}}|\log_2\epsilon|^{\frac{3}{2}(d-1)+1}\big). 
\end{equation*}
This result represents a substantial reduction in network complexity, since the input dimension $d$ enters only through the logarithmic factor $|\log_2\epsilon|$. Their work therefore constitutes an important advance in the approximation of Korobov functions using deep ReLU FNNs. However, despite their strong approximation power, fully connected architectures do not exploit the spatial locality and hierarchical feature extraction that are intrinsic to CNNs.

Mao and Zhou \cite{mao2022approximation} investigated the approximation of Korobov functions using 1D deep ReLU CNNs, with errors measured in the $L^p([0,1]^d)$ norm. In their analysis, the network size depends explicitly on the parameter $p$. More precisely,
the required network size scales as
\begin{equation*}
\mathcal{O}(\epsilon^{-\frac{p}{2p-1}}|\log_2\epsilon|^{\frac{3p-1}{2p-1}(d-1)+2})
\end{equation*}
for $2\leq p\leq\infty$. For the case of $p=\infty$, this complexity is comparable to the result by Montanelli and Du \cite{montanelli2019new}, differing only by a factor of $|\log_2\epsilon|$. In comparison, our result establishes a fully constructive approximation theory for 2D deep
ReLU CNNs on domains $\Omega=[0,1]^{d\times d}$. As shown in Corollary \ref{cor:apprate},
for all $2\le p\le\infty$, an $\varepsilon$-accurate approximation can be achieved with network size
\begin{equation*}
\size(h)=\mathcal{O}\big(\epsilon^{-\frac{1}{2}}|\log_2\epsilon|^{\frac{3(d^2-1)}{2}+1}\big).
\end{equation*}
For $2\leq p<\infty$, this network size bound removes the explicit $p$-dependence appearing in the result of Mao and Zhou, leading to a strictly improved complexity bound with respect to $p$.

It is important to emphasize that all network weights in our construction, including the convolutional kernels, biases, and coefficient vectors, depend continuously on the target function being approximated. To assess the optimality of our results, we compare the resulting network complexity bounds with known lower bounds from the literature. Under the assumption of continuous weight selection, Blanchard and Bennouna \cite{blanchard2021shallow} proved that, with the error measured in the $L^\infty(\Omega)$ norm, any approximation method requires at least
\begin{equation*}
c\epsilon^{-\frac{1}{2}}|\log_2\epsilon|^{\frac{d^2-1}{2}}
\end{equation*}
parameters (for some constant $c>0$) to achieve an $\epsilon$-approximation of all functions from the unit ball of $X^{2,\infty}(\Omega)$. Consequently, up to a logarithmic factor, the network complexity achieved by our construction matches this lower bound. This justifies describing our result as
\textit{near-optimal} in terms of parameter scaling for the space $X^{2,\infty}(\Omega)$. For the space $X^{2,p}(\Omega)$ with $2\le p<\infty$, although the same upper bound holds, the question of whether the network complexity achieved in this work is optimal
remains open.

\section{Proofs of Main Results}\label{sec:proof4thm1}
Recall that, for all $2\leq p\leq \infty$, any function $f\in X^{2,p}(\Omega)$ can be well approximated by its truncated version $f_n^{(1)}$, as described in Subsection \ref{subsec:korov}. The strategy of our proof is to construct a 2D CNN $h_n$ that can accurately represent $f_n^{(1)}$. The main challenge arises from the need to implement the product of all elements of a tensor $X$ in the space $[0,1]^{d\times d}$ through a 2D CNN. This process is crucial for approximating the hierarchical basis functions $\phi_{\bm{l},\bm{i}}$. 

A critical insight from \cite{mao2022approximation} is that in the 1D setting, the product of vector components can be effectively achieved using 1D convolutions by leveraging horizontal shifts (left, right). These two shifts play a pivotal role in operating tensor components and are efficiently implemented by 1D convolutional operations. However, extending this method to 2D CNNs induces additional complexity due to the interplay between horizontal and vertical dimensions. In the 2D setting, shifts can occur in eight different directions: horizontal (left, right), vertical (up, down), and diagonal (top-right, top-left, bottom-right, bottom-left). As a result, the challenge lies in how to utilize 2D convolutions to implement these various shift operations effectively. 

To address this challenge, we introduce basic kernel blocks designed to implement these directional shifts using 2D convolutional operations. Specifically, for any $k\in\mathbb{N}$, we define the basic blocks $S^{s,t}\in \mathbb{R}^{(2k+1)\times (2k+1)}$ for $s,t\in -k:k$ as the matrices with components
\begin{equation*}
[S^{s,t}]_{s^\prime,t^\prime}:=	
\begin{cases}
1, & \text{if}~s^\prime=s~\text{and}~t^\prime=t, \\
0, & \text{otherwise}.
\end{cases}
\end{equation*}
For instance, when $k=1$, the matrices $S^{s,t}$ are as follows   
\begin{equation*}
\begin{aligned}
S^{-1,-1} &= 
\begin{pmatrix}
1 & 0 & 0 \\
0 & 0 & 0 \\
0 & 0 & 0
\end{pmatrix}, &
S^{-1,0} &= 
\begin{pmatrix}
0 & 1 & 0 \\
0 & 0 & 0 \\
0 & 0 & 0
\end{pmatrix}, &
S^{-1,1} &= 
\begin{pmatrix}
0 & 0 & 1 \\
0 & 0 & 0 \\
0 & 0 & 0
\end{pmatrix}, \\
S^{0,-1} &= 
\begin{pmatrix}
0 & 0 & 0 \\
1 & 0 & 0 \\
0 & 0 & 0
\end{pmatrix}, &
S^{0,0} &= 
\begin{pmatrix}
0 & 0 & 0 \\
0 & 1 & 0 \\
0 & 0 & 0
\end{pmatrix}, &
S^{0,1} &= 
\begin{pmatrix}
0 & 0 & 0 \\
0 & 0 & 1 \\
0 & 0 & 0
\end{pmatrix},
\end{aligned}
\end{equation*}

\begin{equation*}
\begin{aligned}
S^{1,-1} &= 
\begin{pmatrix}
0 & 0 & 0 \\
0 & 0 & 0 \\
1 & 0 & 0
\end{pmatrix}, &
S^{1,0} &= 
\begin{pmatrix}
0 & 0 & 0 \\
0 & 0 & 0 \\
0 & 1 & 0
\end{pmatrix}, &
S^{1,1} &= 
\begin{pmatrix}
0 & 0 & 0 \\
0 & 0 & 0 \\
0 & 0 & 1
\end{pmatrix}.
\end{aligned}
\end{equation*}
It is easy to verify that the matrices $S^{1,-1}$, $S^{1,1}$, $S^{-1,-1}$, and $S^{-1,1}$ correspond to diagonal shifts in the directions of top-right, top-left, bottom-right, and bottom-left, respectively. Similarly, $S^{1,0}$ and $S^{-1,0}$ correspond to vertical shifts (up and down), while $S^{0,1}$ and $S^{0,-1}$ correspond to horizontal shifts (left and right).

By incorporating these basic blocks, we obtain the following proposition. It establishes that a specific type of CNN, denoted as $\widetilde{\Pi}_n$, can be constructed to approximate the product of all components from a tensor $X\in [0,1]^{d\times d}$ with a specified error bound. The proof of this proposition is provided in Appendix \ref{app:proof4prod4tensor}.

\begin{proposition}\label{pro:app4prod4tensor}
	Let $k,d\in\mathbb{N}$. For any $n\in\mathbb{N}$, there exists a mapping $\widetilde{\Pi}_n\in\mathcal{C}^{12,L}_{2k+1}(\mathbb{R}^{d\times d},\mathbb{R}^{d\times d})$ with $L=2(2n+3)\cdot\lceil\log_2d\rceil+2(d-1)$ such that
	\begin{equation*}
	\bigg|[\widetilde{\Pi}_n(X)]_{d,d}-\prod_{i,j=1}^d[X]_{i,j}\bigg|\leq 3\cdot 2^{-2n-1}(d^2-1),\quad X\in [0,1]^{d\times d}.
	\end{equation*}
\end{proposition}

For any $\bm{l}\in\mathbb{N}^{d^2}$, let $I_{\bm l}$ and $\phi_{\bm{l},\bm{i}}$ be defined as in Subsection \ref{subsec:korov}, with $D$ replaced by $d^2$ and $\bm{x}$ replaced by the vectorization $\vect(X)$ of the input tensor $X\in[0,1]^{d\times d}$, respectively. From Proposition \ref{pro:app4prod4tensor}, we obtain the following result, which shows that there exists a network, denoted as $g_{\bm{l},\bm{i}}$, capable of approximating hierarchical basis functions $\phi_{\bm{l},\bm{i}}$ in Korobov spaces with controlled accuracy. The proof can be found in Appendix \ref{app:proof4prod4basis}.

\begin{proposition}\label{pro:app4prod4basis}
	Let $k,d\in\mathbb{N}$ with $d\geq 3$, and let $\bm{l}\in\mathbb{N}^{d^2}$. For any $n\in\mathbb{N}$ and $\bm{i}\in I_{\bm l}$, there exists a mapping $\bm{g}_{\bm{l},\bm{i}}\in\mathcal{C}^{2d^2,L}_{2k+1}(\mathbb{R}^{d\times d},\mathbb{R}^{d\times d})$ with $L=2(2n+3)\lceil\log_2d\rceil+5d$ such that
	\begin{equation}\label{eq:errbasisfun}
	\Big|[\bm{g}_{\bm{l},\bm{i}}(X)]_{d,d}-\phi_{\bm{l},\bm{i}}(\vect(X))\Big|\leq \frac{3}{2}\cdot 2^{-2n}(d^2-1),\quad X\in [0,1]^{d\times d}.
	\end{equation}
\end{proposition}

We are now positioned to prove Theorem \ref{thm:approrate}.
\begin{proof}[Proof of Theorem \ref{thm:approrate}]
	For any $n\in\mathbb{N}$, let $f_{n}^{(1)}$ be the truncated approximation of $f$, as described in (\ref{eq:sparseapp}), with $\bm{x}$ replaced by $\vect(X)$, the vectorized form of the tensor $X\in[0,1]^{d\times d}$, and $D$ replaced by $d^2$. Formally,
	\begin{equation}\label{eq:projoff}
	f_n^{(1)}(\vect(X))=\sum_{|\bm{l}|_1\leq n+d^2-1}\sum_{\bm{i}\in I_{\bm{l}}} v_{\bm{l},\bm{i}}\phi_{\bm{l},\bm{i}}(\vect(X)).
	\end{equation}
	
	Let $K$ denote the kernel 
	\begin{equation*}
	K:=
	\begin{pmatrix}
	S^{0,0}\\
	\vdots\\
	S^{0,0}                           
	\end{pmatrix}\in\mathbb{R}^{\theta_n\times 1\times(2k+1)\times(2k+1)},
	\end{equation*}
	where $\theta_n:=\#\Xi_n$ is the cardinality of the set $\Xi_n:=\{(\bm{l},\bm{i}):|\bm{l}|_1\leq n+d^2-1,\bm{i}\in I_{\bm{l}}\}$.
	Using the mappings $\bm{g}_{\bm{l},\bm{i}}$ from Proposition \ref{pro:app4prod4basis}, we define 
	\begin{equation*}
	\bm{g}:=\bigg(\bigoplus_{(\bm{l},\bm{i})\in\Xi_n}\bm{g}_{\bm{l},\bm{i}}\bigg)\circ\sigma\circ A_{K},
	\end{equation*}	
	where $\oplus$ is the concatenation to be defined in Lemma \ref{lem:con_net} in the appendix. By Proposition \ref{pro:app4prod4basis},  $\bm{g}\in\mathcal{C}^{W,L}_{2k+1}(\mathbb{R}^{d\times d},\mathbb{R}^{\theta_n\times d\times d})$ with $W=2\theta_nd^2$ and $L=2(2n+3)\lceil\log_2d\rceil+6d$, and the size of $\bm{g}$ is bounded above as
	\begin{align*}
	\size(\bm{g})&\leq\sum_{(\bm{l},\bm{i})\in\Xi_n}\size(\bm{g}_{\bm{l},\bm{i}})+\size(\sigma\circ A_{K})\\
	&\leq\theta_n\cdot(2k+1)^2\cdot 4d^4 \big(2(2n+3)\lceil\log_2d\rceil+5d\big)+\theta_n\cdot\big((2k+1)^2+1\big).
	\end{align*}
		
	Let $\mu$ be a bijection from the set $1:\theta_n$ to the set $\Xi_n$. Then, for $c\in 1:\theta_n$ and $X\in[0,1]^{d\times d}$, we have 
	\begin{equation}\label{eq:che2prod}
	[\bm{g}(X)]_{c,d,d}=[\bm{g}_{\mu(c)}(X)]_{d,d}.
	\end{equation}
	We define the vector $\bm{\alpha}\in\mathbb{R}^{\theta_nd^2}$ by
	\begin{equation*}
	\alpha_i:= 
	\begin{cases}
	v_{\mu(c)}, & \text{if}~i=cd^2~\text{for some}~c\in 1:\theta_n, \\ 
	0,     & \text {otherwise}.
	\end{cases}
	\end{equation*}
	Then, using this vector, we construct a function $h_n$ in the space $\mathcal{H}^{W,L}_{2k+1}(\mathbb{R}^{d\times d})$ as follows
	\begin{equation*}
	h_n(X):=\sum_{i=1}^{\theta_nd^2}\alpha_i[\vect(\bm{g}(X))]_i,\quad X\in[0,1]^{d\times d}.
	\end{equation*}	
	It follows from (\ref{eq:che2prod}) that for $X\in[0,1]^{d\times d}$,
	\begin{equation}\label{eq:appproof}
	h_n(X)=\sum_{c=1}^{\theta_n}v_{\mu(c)}[\bm{g}_{\mu(c)}(X)]_{d,d}=\sum_{|\bm{l}|_1\leq n+d^2-1}\sum_{\bm{i}\in I_{\bm{l}}} v_{\bm{l},\bm{i}}[\bm{g}_{\bm{l},\bm{i}}(X)]_{d,d}.
	\end{equation}
	Moreover, the size of $h_n$ is bounded as
	\begin{align*}
	\size(h_n)&\leq\size(\bm{g})+\theta_n\\
	&\leq\theta_n\cdot(2k+1)^2\cdot 4d^4 \big(2(2n+3)\lceil\log_2d\rceil+5d\big)+\theta_n\cdot\big((2k+1)^2+2\big)\\
	&\leq 24(2k+1)^2d^5n\theta_n.
	\end{align*}
	
	Next, we estimate the approximation error in \eqref{eq:sparseappnonconst}, making the implicit constant explicit. We claim that, for $2\le p\le\infty$,	
	\begin{equation*} 
	\Big\|f(\vect(X))-f_n^{(1)}(\vect(X))\Big\|_{L^{p}(\Omega)} \leq 2^{-2d^2-2n}n^{d^2-1}. 
	\end{equation*}
	Indeed, since $\|f\|_{X^{2,p}(\Omega)}\leq 1$ and the family $\{\phi_{\bm{l},\bm{j}}:\bm{j}\in I_{\bm{l}}\}$ consists of functions with pairwise disjoint supports, it follows from (\ref{eq:normofbasis}) and (\ref{eq:bou4cof}) that, for $2\le p<\infty$,	
	\begin{eqnarray*}
	\bigg\|\sum_{\bm{j}\in I_{\bm{l}}}v_{\bm{l},\bm{j}}\phi_{\bm{l},\bm{j}}\bigg\|_{L^{p}(\Omega)}^p&=&\sum_{\bm{j}\in I_{\bm{l}}}\big\|v_{\bm{l},\bm{j}}\phi_{\bm{l},\bm{j}}(\vect(X))\big\|_{L^{p}(\Omega)}^p\\
	&\leq&\sum_{\bm{j}\in I_{\bm{l}}}2^{-p|\bm{l}|_1-pd^2}\|\phi_{\bm{l},\bm{j}}\|_{L^q(\Omega)}^p\,\|f|_{\operatorname{supp}(\phi_{\bm{l},\bm{j}})}\|_{X^{2,p}(\Omega)}^p\|\phi_{\bm{l},\bm{j}}\|_{L^p(\Omega)}^p\\
	&\leq&2^{-pd^2}\cdot2^{-|\bm{l}|_1(p+p/q+1)}\bigg(\frac{2}{q+1}\bigg)^{pd^2/q}\bigg(\frac{2}{p+1}\bigg)^{d^2}, 
	\end{eqnarray*}
	where $q$ denotes the conjugate exponent of $p$. Consequently,
	\begin{eqnarray*}
	\bigg\|\sum_{\bm{j}\in I_{\bm{l}}}v_{\bm{l},\bm{j}}\phi_{\bm{l},\bm{j}}\bigg\|_{L^{p}(\Omega)}\leq 2^{-d^2-2|\bm{l}|_1}.
	\end{eqnarray*}
    Therefore,
	\begin{equation*}
	\Big\|f(\vect(X))-f_n^{(1)}(\vect(X))\Big\|_{L^{p}(\Omega)}\leq2^{-d^2}\sum_{|\bm{l}|_1>n+d^2-1} 2^{-2|\bm{l}|_1},
	\end{equation*}
    and the same bound holds for $p=\infty$, obtained by an analogous argument.
	
	To estimate the sum $\sum_{|\bm{l}|_1>n+d^2-1} 2^{-2|\bm{l}|_1}$, observe that
	\begin{equation*}
	\sum_{|\bm{l}|_1>n+d^2-1} 2^{-2|\bm{l}|_1}=\sum_{l>n+d^2-1}2^{-2l}\binom{l-1}{d^2-1}=2^{-2n-2d^2}\sum_{l=0}^{\infty}2^{-2l}\binom{l+n+d^2-1}{d^2-1}.
	\end{equation*}   
	Applying the first equality in Lemma \ref{lem:boucomb} in the appendix with $x=2^{-2}$, we obtain
	\begin{equation}\label{eq:geneapp}
	\sum_{|\bm{l}|_1>n+d^2-1} 2^{-2|\bm{l}|_1}\leq \frac{4}{3}\cdot2^{-2n-2d^2}\cdot\sum_{l=0}^{d^2-1}\binom{n+d^2-1}{l}\bigg(\frac{1}{3}\bigg)^{d^2-1-l}.
	\end{equation}
	Moreover, for $n\geq d^2-1$,
	\begin{equation*}
	\sum_{l=0}^{d^2-1}\binom{n+d^2-1}{l}\bigg(\frac{1}{3}\bigg)^{d^2-1-l}\leq\binom{n+d^2-1}{d^2-1}\sum_{l=0}^{d^2-1}\bigg(\frac{1}{3}\bigg)^{d^2-1-l}\leq\frac{3}{2}\binom{n+d^2-1}{d^2-1}.
	\end{equation*}
	Substituting this estimate into \eqref{eq:geneapp} yields
	\begin{equation*}
	\sum_{|\bm{l}|_1>n+d^2-1} 2^{-2|\bm{l}|_1}\leq 2^{1-2n-2d^2}\binom{n+d^2-1}{d^2-1}.
	\end{equation*}
	Using the bound $\binom{n+d^2-1}{d^2-1}\leq (2n)^{d^2-1}$, we conclude
	\begin{equation*}
	\sum_{|\bm{l}|_1>n+d^2-1} 2^{-2|\bm{l}|_1}\leq 2^{-d^2-2n}n^{d^2-1},
	\end{equation*}
	and consequently,
	\begin{equation*}
	\Big\|f(\vect(X))-f_n^{(1)}(\vect(X))\Big\|_{L^{p}(\Omega)}\leq 2^{-2d^2-2n}n^{d^2-1},
	\end{equation*}
	 which establishes the claimed bound.
	 
	 As noted in the proof of Proposition \ref{pro:app4prod4basis}, the support of $[\bm{g}_{\bm{l},\bm{j}}(X)]_{d,d}$ is contained in the support of $\phi_{\bm{l},\bm{j}}(\vect(X))$, thus it follows from (\ref{eq:errbasisfun}) that, for $2\leq p<\infty$, 
	\begin{equation*}
	\bigg\|\sum_{\bm{j}\in I_{\bm{l}}}v_{\bm{l},\bm{j}}\big(\phi_{\bm{l},\bm{j}}(\vect(X))-[\bm{g}_{\bm{l},\bm{j}}(X)]_{d,d}\big)\bigg\|_{L^p(\Omega)}
	\leq\frac{3}{2}(d^2 - 1) 2^{-d^2-2n}\,2^{-|\bm{l}|_1(1+1/q)}.
	\end{equation*}	 
    Consequently, by \eqref{eq:projoff} and \eqref{eq:appproof},
	\begin{eqnarray*}
	\Big\|f_n^{(1)}-h_n\Big\|_{L^p(\Omega)}
	&\leq&\sum_{|\bm{l}|_1\leq n+d^2-1}\frac{3}{2}(d^2 - 1) 2^{-d^2-2n}\cdot 2^{-|\bm{l}|_1(1+1/q)}\\
	&\leq&\frac{3}{2}(d^2-1)2^{-d^2(2+1/q)-2n}\sum_{l=0}^{n-1} 2^{-l(1+1/q)} \binom{l+d^2-1}{d^2-1}.
	\end{eqnarray*}	
    By applying the second equality in Lemma \ref{lem:boucomb} with $x=2^{-(1+1/q)}$, we obtain
	\begin{equation*}
	\Big\|f_n^{(1)}(\vect(X))-h_n(\vect(X))\Big\|_{L^p(\Omega)}\leq\frac{3}{2}(d^2-1)2^{-2n-d^2},
	\end{equation*}	
	and an analogous argument yields the same bound for the case $p=\infty$.	Combining this estimate with the bound for $\big\|f-f_n^{(1)}\big\|_{L^p(\Omega)}$, we arrive at
	\begin{equation}\label{eq:errtotinf}
	\big\|f-h_n\big\|_{L^p(\Omega)}\leq2^{2-d^2-2n}n^{d^2-1}, \qquad 2\le p\le\infty.
	\end{equation}	
		
	For any $N\in\mathbb{N}$, define
	\begin{equation*}
	\tau_N:=\max\bigg\{n\in\mathbb{N}:\sum_{|\bm{l}|_1\leq n+d^2-1}\# I_{\bm{l}}\leq N\bigg\}.
	\end{equation*}
	By definition, we have $\theta_{\tau_N}\leq N$. According to \cite[Lemma 3.6]{bungartz2004sparse}, $\tau_N$ satisfies
	\begin{equation*}
	\log_2\left(\frac{N}{(\log_2N)^{d^2-1}}\right)\leq \tau_N\leq \log_2N.
	\end{equation*}
	Substituting $n=\tau_N$ into the inequality (\ref{eq:errtotinf}), we obtain, for $N\geq\theta_{d^2-1}$,
	\begin{equation*}
	\big\|f-h_{\tau_N}\big\|_{L^p(\Omega)}\leq \frac{4}{2^{d^2}}\frac{\big(\log_2N\big)^{3(d^2-1)}}{N^2}.
	\end{equation*}
	Moreover, $h_{\tau_N}$ belongs to the hypothesis class $\mathcal{H}^{W,L}_{2k+1}(\mathbb{R}^{d\times d})$ with 
	\begin{equation*}
	W=2Nd^2,\qquad L=2(2\lceil\log_2N\rceil+3)\lceil\log_2d\rceil+6d,
	\end{equation*}
    and its network size satisfies 
	\begin{equation*}
	\size(h_{\tau_N})\leq 24(2k+1)^2d^5N\log_2N.
	\end{equation*}	
	This completes the proof of Theorem \ref{thm:approrate}.
\end{proof}

\section{Conclusion}\label{sec:conclu}
We have introduced a constructive framework based on basic kernel blocks and multi-channel architectures to establish explicit upper bounds on the network complexity required for approximating Korobov functions by 2D deep ReLU CNNs. Our analysis shows that 2D CNNs can approximate such functions efficiently, substantially alleviating the curse of dimensionality inherent in high-dimensional approximation problems. Furthermore, under the continuous weight selection model, the resulting complexity bounds are shown to be nearly optimal. Overall, this work lays a foundation for approximation theory in 2D CNN-based deep learning models, which contributes to better understanding of their generalization properties.

CNNs with periodic padding are also of independent interest, particularly in the context of operator learning for numerical partial differential equations with periodic boundary conditions.
Although the large-kernel decomposition theorem in \cite{he2022approximation} remains valid for periodic padding, the main approximation results of the present paper do not directly extend to this setting. In particular, the proof of Theorem \ref{thm:approrate} relies crucially on Proposition \ref{pro:app4prod4basis}, whose argument in turn depends on Lemma \ref{lem:takcomp} in Appendix \ref{app:proof4prod4basis}. This lemma exploits localization and boundary-decoupling properties that are specific to the zero-padding setting and fail under periodic padding. 
Consequently, extending Proposition \ref{pro:app4prod4basis}, and hence Theorem \ref{thm:approrate} as well as Corollary \ref{cor:apprate}, to CNNs with periodic padding would require new analytical tools or alternative constructions, which we leave for future investigations.

The CNN architectures studied in this paper are naturally contained as special cases of ResNet-type networks (by disabling skip connections). Consequently, our approximation results extend directly to ResNet-type architectures in terms of representational capabilities. Regarding activation functions, Zhang et al. \cite{zhang2024deep} showed that ReLU networks can be approximated, on bounded domains, by networks with smooth activations such as Sigmoid, Tanh, Swish, or GELU, with only moderate increases in width and depth. These results suggest that the approximation rates established here should also hold for CNNs with such activations. Nevertheless, our analysis is fully constructive and relies essentially on the piecewise-linear nature of ReLU, particularly in the construction of hierarchical basis functions.

The constructive technique developed here, centered on directional kernel blocks, provides a principled pathway for extending the approximation theory to CNNs in higher dimensions. The central idea—implementing global tensor operations through compositions of localized directional shifts—is, in principle, applicable beyond two dimensions. However, the complexity of such extensions grows rapidly with the dimension. For example, in an $n$-dimensional setting with small kernel sizes, the number of directional shifts that must be coordinated increases exponentially with $n$, leading to substantially more intricate network constructions and complexity analyses. Therefore, our detailed resolution of the 2D case serves a critical dual purpose: It resolves the approximation problem for the practically important class of image-like data, and it establishes the essential methodological foundation and a concrete prototype for future theoretical investigations of higher-dimensional CNNs. In this sense, the transition from 1D sequences to 2D spatial grids represents a pivotal and nontrivial step in the hierarchy of constructive CNN approximation theory.

\section*{Acknowledgments}
The authors would like to thank the anonymous reviewers and the editor for their valuable comments and suggestions, which helped to improve the quality of this paper.

\bibliographystyle{plain}
\bibliography{2dCNN_Korobov_ref.bib}

\begin{appendices}

\section{Basic CNN Constructions}\label{app:basiccon}
In this section, we collect some important 2D deep ReLU CNN constructions which will be used repeatedly to construct more complex networks.

\begin{lemma}[Widening CNNs]\label{lem:wid_cnn}
	Let $k,d,W_1,W_2,L,c_0,c_L\in\mathbb{N}$ with $W_1\leq W_2$. For $L\geq 2$, the following inclusion holds: 
	\begin{equation*}
	\mathcal{C}^{W_1,L}_{2k+1}(\mathbb{R}^{c_0\times d\times d},\mathbb{R}^{c_L\times d\times d})\subset \mathcal{C}^{W_2,L}_{2k+1}(\mathbb{R}^{c_0\times d\times d},\mathbb{R}^{c_L\times d\times d}).
	\end{equation*}
\end{lemma}
\begin{proof}
	Let $\bm{f}\in \mathcal{C}^{W_1,L}_{2k+1}(\mathbb{R}^{c_0\times d\times d},\mathbb{R}^{c_L\times d\times d})$. By definition, we have
	\begin{equation*}
	\bm{f}=\sigma\circ A_{K^L,\bm{b}^L}\circ\cdots\circ\sigma\circ A_{K^1,\bm{b}^1},
	\end{equation*} 
	where $K^l\in\mathbb{R}^{c_{l}\times c_{l-1}\times (2k+1)\times (2k+1)}$ and $\bm{b}^l\in\mathbb{R}^{c_l}$ for $l\in 1:L$. We extend the first layer to a larger width $W_2\geq W_1$. Define
	\begin{equation*}
	\scalebox{0.9}{$
	\tilde{K}^1:=\begin{pmatrix}
	K^1\\
	\bm{0}_{(W_2-c_1)\times c_0\times(2k+1)\times(2k+1)}
	\end{pmatrix}	
	\in \mathbb{R}^{W_2\times c_0\times (2k+1)\times (2k+1)},~~\tilde{\bm{b}}^1:=\begin{pmatrix}
	\bm{b}^1\\
	\bm{0}_{W_2-c_1}
	\end{pmatrix}\in\mathbb{R}^{W_2}.
	$}
	\end{equation*}
	For any input $X\in \mathbb{R}^{c_0\times d\times d}$, we have
	\begin{equation*}
	\sigma\circ A_{\tilde{K}^1,\tilde{\bm{b}}^1}(X)=\begin{pmatrix}
	\sigma\circ A_{K^1,\bm{b}^1}(X)\\
	\bm{0}_{(W_2-c_1)\times d\times d}
	\end{pmatrix}
	\in\mathbb{R}^{W_2\times d\times d}.
	\end{equation*}	
	Next, set 
	\begin{equation*}
	\tilde{K}^2=\begin{pmatrix}
	K^2,\bm{0}_{c_2\times (W_2-c_1)\times (2k+1)\times (2k+1)}
	\end{pmatrix}\in \mathbb{R}^{c_2\times W_2\times (2k+1)\times (2k+1)}~~\text{and}~~\tilde{\bm{b}}^2=\bm{b}^2\in \mathbb{R}^{c_2}.
	\end{equation*}
	We further obtain
	\begin{equation*}
	\sigma\circ A_{\tilde{K}^2,\tilde{\bm{b}}^2}\circ\sigma\circ A_{\tilde{K}^1,\tilde{\bm{b}}^1}(X)=\sigma\circ A_{K^2,\bm{b}^2}\circ\sigma\circ A_{K^1,\bm{b}^1}(X).
	\end{equation*}	
	Let $\bm{g}=\sigma\circ A_{K^L,\bm{b}^L}\circ\cdots\circ\sigma\circ A_{K^3,\bm{b}^3}\circ\sigma\circ A_{\tilde{K}^2,\tilde{\bm{b}}^2}\circ\sigma\circ A_{\tilde{K}^1,\tilde{\bm{b}}^1}$, then
	\begin{equation*}
	\bm{f}=\bm{g}\in \mathcal{C}^{W_2,L}_{2k+1}(\mathbb{R}^{c_0\times d\times d},\mathbb{R}^{c_L\times d\times d}), 
	\end{equation*} 
	from which the desired result follows.	
\end{proof}

\begin{lemma}[Deepening CNNs]\label{lem:dep_cnn}
	Let $k,d,W,L_1,L_2,c_0,c_{L_1}\in\mathbb{N}$ with $L_1\leq L_2$. The following inclusion holds: 
	\begin{equation*}
	\mathcal{C}^{W,L_1}_{2k+1}(\mathbb{R}^{c_0\times d\times d},\mathbb{R}^{c_{L_1}\times d\times d})\subset \mathcal{C}^{W,L_2}_{2k+1}(\mathbb{R}^{c_0\times d\times d},\mathbb{R}^{c_{L_1}\times d\times d}).
	\end{equation*}
\end{lemma}
\begin{proof}
	Let $S^{0,0}\in \mathbb{R}^{(2k+1)\times (2k+1)}$ be the basic block as defined in Section \ref{sec:proof4thm1}. For $l\in (L_1+1):L_2$, take 
	\begin{equation*}
	K^l:=\begin{pmatrix}
	S^{0,0} &        &        &        \\
	&S^{0,0} &        &        \\
	&        &\ddots  &        \\
	&        &        & S^{0,0}
	\end{pmatrix}\in\mathbb{R}^{c_{L_1}\times c_{L_1}\times (2k+1)\times (2k+1)},
	~~\bm{b}^l:=\bm{0}_{c_{L_1}}\in\mathbb{R}^{c_{L_1}}.
	\end{equation*}
	Then, for each $\bm{f}\in \mathcal{C}^{W,L_1}_{2k+1}(\mathbb{R}^{c_0\times d\times d},\mathbb{R}^{c_{L_1}\times d\times d})$, we have 
	\begin{equation*}
	\bm{f}=\sigma\circ A_{K^{L_2},\bm{b}^{L_2}}\circ\cdots\circ\sigma\circ A_{K^{L_1+1},\bm{b}^{L_1+1}}\circ \bm{f}\in \mathcal{C}^{W,L_2}_{2k+1}(\mathbb{R}^{c_0\times d\times d},\mathbb{R}^{c_{L_1}\times d\times d}),
	\end{equation*}
	which proves the claim. 	
\end{proof}

\begin{lemma}[Composing CNNs]\label{lem:comp_net}
	Let $k,d,c,c^{\prime},c^{\prime\prime},W_1,W_2,L_1,L_2\in\mathbb{N}$. Suppose $\bm{f}\in\mathcal{C}^{W_1,L_1}_{2k+1}(\mathbb{R}^{c\times d\times d},\mathbb{R}^{c^{\prime}\times d\times d})$ and $\bm{g}\in \mathcal{C}^{W_2,L_2}_{2k+1}(\mathbb{R}^{c^{\prime}\times d\times d},\mathbb{R}^{c^{\prime\prime}\times d\times d})$. The composition mapping $\bm{g}\circ\bm{f}$ satisfies 
	\begin{equation*}
	\bm{g}\circ\bm{f}\in\mathcal{C}^{W,L}_{2k+1}(\mathbb{R}^{c\times d\times d},\mathbb{R}^{c^{\prime\prime}\times d\times d}),
	\end{equation*}
	where $W=\max\{W_1,W_2\}$ and $L=L_1+L_2$, and the size is given by $\size(\bm{g}\circ\bm{f})=\size(\bm{g})+\size(\bm{f})$.
\end{lemma}
\begin{proof}
	The mapping $\bm{f}$ has the form $\bm{f}=\sigma\circ A_{K^{L_1},\bm{b}^{L_1}}\circ\cdots\circ\sigma\circ A_{K^1,\bm{b}^1}$, where $K^l\in\mathbb{R}^{c_{l}\times c_{l-1}\times (2k+1)\times (2k+1)}$ and $\bm{b}^l\in\mathbb{R}^{c_l}$ for $l\in 1:L_1$, $c_0=c$, $c_{L_1}=c^{\prime}$, and $W_1=\max\{c_0,c_1,\ldots,c_{L_1}\}$. Similarly, $\bm{g}=\sigma\circ A_{\bar{K}^{L_2},\bar{\bm{b}}^{L_2}}\circ\cdots\circ\sigma\circ A_{\bar{K}^1,\bar{\bm{b}}^1}$, where $\bar{K}^l\in\mathbb{R}^{\bar{c}_{l}\times \bar{c}_{l-1}\times (2k+1)\times (2k+1)}$ and $\bar{\bm{b}}^l\in\mathbb{R}^{\bar{c}_l}$ for $l\in 1:L_2$, $\bar{c}_0=c^{\prime}$, $\bar{c}_{L_2}=c^{\prime\prime}$, and $W_2=\max\{\bar{c}_0,\bar{c}_1,\ldots,\bar{c}_{L_2}\}$. Thus, for the composition $\bm{g}\circ\bm{f}$, we have
	\begin{equation*}
	\bm{g}\circ\bm{f}=\sigma\circ A_{\bar{K}^{L_2},\bar{\bm{b}}^{L_2}}\circ\cdots\circ A_{\bar{K}^1,\bar{\bm{b}}^1}\circ\sigma\circ A_{K^{L_1},\bm{b}^{L_1}}\circ\cdots\circ A_{K^1,\bm{b}^1}\in\mathcal{C}^{W,L}_{2k+1}(\mathbb{R}^{c\times d\times d},\mathbb{R}^{c^{\prime\prime}\times d\times d})
	\end{equation*}
	with $W=\max\{W_1,W_2\}$ and $L=L_1+L_2$. 	
\end{proof}

\begin{lemma}[Concatenating CNNs]\label{lem:con_net}
	Let $k,d,c_0,c_L,W,L\in\mathbb{N}$. Suppose $\bm{f},\bm{g}\in \mathcal{C}^{W,L}_{2k+1}(\mathbb{R}^{c_0\times d\times d},\mathbb{R}^{c_L\times d\times d})$. The concatenation mapping
	\begin{equation*}
	\bm{f}\oplus\bm{g}:\mathbb{R}^{2c_0\times d\times d}\rightarrow \mathbb{R}^{2c_L\times d\times d},~~\begin{pmatrix}
	X\\
	Y
	\end{pmatrix}\mapsto \begin{pmatrix}
	\bm{f}(X)\\
	\bm{g}(Y)
	\end{pmatrix}
	\end{equation*}
	satisfies $\bm{f}\oplus\bm{g}\in \mathcal{C}^{2W,L}_{2k+1}(\mathbb{R}^{2c_0\times d\times d},\mathbb{R}^{2c_L\times d\times d})$, and the size is given by $\size(\bm{f}\oplus\bm{g})=\size(\bm{f})+\size(\bm{g})$.
\end{lemma}
\begin{proof}
	Recall that $\bm{f}$ and $\bm{g}$ can be formulated as follows:  
	\begin{equation*}
	\bm{f}=\sigma\circ A_{\tilde{K}^{L},\tilde{\bm{b}}^{L}}\circ\cdots\circ\sigma\circ A_{\tilde{K}^1,\tilde{\bm{b}}^1},~~~\bm{g}=\sigma\circ A_{\bar{K}^{L},\bar{\bm{b}}^{L}}\circ\cdots\circ\sigma\circ A_{\bar{K}^1,\bar{\bm{b}}^1}, 
	\end{equation*}
	where $\tilde{K}^l\in\mathbb{R}^{\tilde{c}_{l}\times \tilde{c}_{l-1}\times (2k+1)\times (2k+1)}$, $\bar{K}^l\in\mathbb{R}^{\bar{c}_{l}\times \bar{c}_{l-1}\times (2k+1)\times (2k+1)}$, $\tilde{\bm{b}}^l\in\mathbb{R}^{\tilde{c}_l}$, and $\bar{\bm{b}}^l\in\mathbb{R}^{\bar{c}_l}$. For $l\in 1:L$, let
	\begin{equation*}
	K^l:=\begin{pmatrix}
	\tilde{K}^{l}&\\
	&\bar{K}^{l}
	\end{pmatrix}\in\mathbb{R}^{(\tilde{c}_{l}+\bar{c}_{l})\times (\tilde{c}_{l-1}+\bar{c}_{l-1})\times (2k+1)\times (2k+1)},
	~~\bm{b}^l:=\begin{pmatrix}
	\tilde{\bm{b}}^{l}\\
	\bar{\bm{b}}^{l}
	\end{pmatrix}
	\in\mathbb{R}^{\tilde{c}_l+\bar{c}_l}.
	\end{equation*}
	Then, by Lemma \ref{lem:wid_cnn}, we have
	\begin{equation*}
	\bm{f}\oplus\bm{g}=\sigma\circ A_{K^{L},\bm{b}^{L}}\circ\cdots\circ\sigma\circ A_{K^1,\bm{b}^1}\in \mathcal{C}^{2W,L}_{2k+1}(\mathbb{R}^{2c_0\times d\times d},\mathbb{R}^{2c_L\times d\times d}),
	\end{equation*}
	as claimed.
\end{proof}

\section{Proof of Proposition \ref{pro:app4prod4tensor}}\label{app:proof4prod4tensor}
Let $g:[0,1]\rightarrow[0,1]$ denote the hat function defined by
\begin{equation*}
g(x):=2\sigma(x)-4\sigma(x-1/2)+2\sigma(x-1),\quad \text{for}~x\in[0,1].
\end{equation*}
For any $m\in\mathbb{N}$, we define the iterated function $g_m:[0,1]\rightarrow[0,1]$ as the composition of $g$
applied $m$ times:
\begin{equation*}
g_m(x):=\underbrace{g\circ g\circ\cdots \circ g}_{m}(x).
\end{equation*}
It has been demonstrated in \cite{yarotsky2017error} that for $x\in[0,1]$, the following equality holds  
\begin{equation*}
x^2=x-\sum_{m=1}^{\infty}4^{-m}g_m(x).
\end{equation*}
For any $n\in\mathbb{N}$, let $\sq_n:[0,1]\rightarrow[0,1]$ denote the function defined by
\begin{equation*}
\sq_n(x)=x-\sum_{m=1}^n4^{-m}g_m(x),\quad \text{for}~x\in[0,1]. 
\end{equation*}
According to \cite{yarotsky2017error}, $\sq_{n}(x)$ is the piecewise linear interpolant of $f(x)=x^2$ with $2^n+1$ uniformly distributed breakpoints $\frac{0}{2^n},\frac{1}{2^n},\ldots,\frac{2^n}{2^n}$. Moreover, for any $x\in[0,1]$,
the difference between $\sq_n(x)$ and $x^2$ satisfies
\begin{equation*}
\sq_n(x)-x^2\in [0,4^{-(n+1)}].
\end{equation*}
We extend $\sq_n$ to a mapping from $[0,1]^{c\times d\times d}$ to $[0,1]^{c\times d\times d}$ by applying $\sq_n$ element-wise to each component of the input tensor. For this extended mapping, we have the following lemma.

\begin{lemma}\label{lem:squappr4tensor}
	Let $k,d,c\in\mathbb{N}$. For any $n\in\mathbb{N}$, the mapping $\sq_n$ belongs to the class $\mathcal{C}^{4c,2(n+1)}_{2k+1}(\mathbb{R}^{c\times d\times d},\mathbb{R}^{c\times d\times d})$ and satisfies the condition
	\begin{equation*}
	\sq_n(X)-X\odot X \in [0,4^{-n-1}]^{c\times d\times d},\quad X\in[0,1]^{c\times d\times d},
	\end{equation*}
	where $\odot$ denotes the Hadamard (component-wise) product of tensors.
\end{lemma}
\begin{proof}	    
	It suffices to show that  $\sq_n\in\mathcal{C}^{4c,2(n+1)}_{2k+1}(\mathbb{R}^{c\times d\times d},\mathbb{R}^{c\times d\times d})$. To this end, we introduce a mapping $\bm{f}^n:[0,1]^{c\times d\times d}\rightarrow[0,1]^{2c\times d\times d}$, defined by
	\begin{equation*}
	\bm{f}^n(X)=\begin{pmatrix}
	\sq_n(X)\\
	g_n(X)
	\end{pmatrix},\quad \text{for}~X\in[0,1]^{c\times d\times d}.
	\end{equation*}
	
	We assert that $\bm{f}^n$ belongs to the class $\mathcal{C}^{4c,2n+1}_{2k+1}(\mathbb{R}^{c\times d\times d},\mathbb{R}^{2c\times d\times d})$. We prove this assertion by induction on $n$. For the base case $n=1$, consider the following kernel $K^0$ and bias $\bm{b}^0$
	\begin{equation*}
	K^0:=
	\begin{pmatrix}
	S^{0,0}    &          &          \\
	& \ddots   &          \\
	&          &  S^{0,0} \\
	S^{0,0}    &          &          \\
	& \ddots   &          \\
	&          &  S^{0,0}                
	\end{pmatrix}
	\in\mathbb{R}^{2c\times c \times(2k+1)\times(2k+1)},~~\bm{b}^0:= \bm{0}_{2c}\in\mathbb{R}^{2c},
	\end{equation*}
	where $S^{0,0}\in\mathbb{R}^{(2k+1)\times(2k+1)}$ is the basic block defined in Section \ref{sec:proof4thm1}.
	With $K^0$ and $\bm{b}^0$, we duplicate $X\in[0,1]^{c\times d\times d}$ as follows:
	\begin{equation*} 
	\bm{f}^{0}(X):=
	\begin{pmatrix}
	X\\
	X
	\end{pmatrix}=\sigma\circ A_{K^0,\bm{b}^0}(X)\in\mathbb{R}^{2c\times d\times d}.
	\end{equation*}		
	Next, we define the kernel $K^{1,1}\in\mathbb{R}^{4c\times 2c \times(2k+1)\times(2k+1)}$ and bias $\bm{b}^{1,1}\in\mathbb{R}^{4c}$ as
	\begin{equation*}
	K^{1,1}:=
	\begin{pmatrix}
	S^{0,0} &          &           &            &          &        \\ 
	& \ddots   &           &            &          &        \\
	&          & S^{0,0}   &            &          &        \\
	&          &           &  S^{0,0}   &          &        \\
	&          &           &            &  \ddots  &        \\ 
	&          &           &            &          &S^{0,0} \\
	&          &           &  S^{0,0}   &          &        \\
	&          &           &            &  \ddots  &        \\ 
	&          &           &            &          &S^{0,0} \\
	&          &           &  S^{0,0}   &          &        \\
	&          &           &            &  \ddots  &        \\ 
	&          &           &            &          &S^{0,0}               
	\end{pmatrix}
	,~~\bm{b}^{1,1}:=-\begin{pmatrix}
	\bm{0}_c\\
	\bm{0}_c\\
	(\frac{\bm{1}}{\bm 2})_{c}\\
	\bm{1}_{c}
	\end{pmatrix}.
	\end{equation*}
	Then, $\bm{f}^{1,1}:=\sigma\circ A_{K^{1,1},\bm{b}^{1,1}}\circ\bm{f}^{0}\in \mathcal{C}^{4c,2}_{2k+1}(\mathbb{R}^{c\times d\times d},\mathbb{R}^{4c\times d\times d})$, and a direct computation gives for any $X\in[0,1]^{c\times d\times d}$,
	\begin{equation*} 
	\bm{f}^{1,1}(X)=
	\begin{pmatrix}
	X\\
	\sigma(X)\\
	\sigma(X-(\frac{\bm{1}}{\bm 2})_{c}\bm{1}_{d\times d})\\
	\sigma(X-\bm{1}_{c}\bm{1}_{d\times d})
	\end{pmatrix}\in\mathbb{R}^{4c\times d\times d}.
	\end{equation*}	
	We further define the kernel $K^{1,2}\in\mathbb{R}^{2c\times 4c\times (2k+1)\times (2k+1)}$ as 
	\begin{equation*}
	\setcounter{MaxMatrixCols}{20}
	\scalebox{0.9}{$
	K^{1,2} := \begin{pmatrix}
	S^{0,0} &      && -\frac{S^{0,0}}{2}&      && S^{0,0} &        &&-\frac{S^{0,0}}{2}& & \\
	&\ddots&&               &\ddots&&         &\ddots  &&              &\ddots & \\
	&       &S^{0,0} &  &  &-\frac{S^{0,0}}{2} & &  & S^{0,0} &  & & -\frac{S^{0,0}}{2}\\
	&       && 2S^{0,0}  &  &&-4S^{0,0}  &  &  & 2S^{0,0} & & \\
	&       &&        &\ddots  &&  &\ddots  &  &  &\ddots & \\
	&       &&        &  &2S^{0,0}   &  &  &-4S^{0,0}   &  && 2S^{0,0} 
	\end{pmatrix}
	$}
	\end{equation*}		
	and the bias $\bm{b}^{1,2}\in\mathbb{R}^{2c}$ as $\bm{b}^{1,2}:=\bm{0}_{2c}$. According to the definition of $g$, we have for any $X\in[0,1]^{c\times d\times d}$,
	\begin{equation*} 
	\bm{f}^1(X):=
	\begin{pmatrix}
	\sq_{1}(X)\\
	g_{1}(X)
	\end{pmatrix}=\sigma\circ A_{K^{1,2},\bm{b}^{1,2}}\circ \bm{f}^{1,1}(X),
	\end{equation*}	
	which, together with $\bm{f}^{1,1}\in \mathcal{C}^{4c,2}_{2k+1}(\mathbb{R}^{c\times d\times d},\mathbb{R}^{4c\times d\times d})$, implies
	\begin{equation*}
	\bm{f}^1\in\mathcal{C}^{4c,3}_{2k+1}(\mathbb{R}^{c\times d\times d},\mathbb{R}^{2c\times d\times d}).
	\end{equation*}
	Thus, the assertion holds for $n=1$. 
	
	For the inductive step, assume for some $n\geq 1$, $\bm{f}^n\in\mathcal{C}^{4c,2n+1}_{2k+1}(\mathbb{R}^{c\times d\times d},\mathbb{R}^{2c\times d\times d})$. We need to show that
	\begin{equation*}
	\bm{f}^{n+1}\in\mathcal{C}^{4c,2(n+1)+1}_{2k+1}(\mathbb{R}^{c\times d\times d},\mathbb{R}^{2c\times d\times d}).
	\end{equation*}
	First, we define the kernel $K^{n+1,1}\in\mathbb{R}^{4c\times 2c \times(2k+1)\times(2k+1)}$ and bias $\bm{b}^{n+1,1}\in\mathbb{R}^{4c}$ as
	\begin{equation*}
	K^{n+1,1}:=
	\begin{pmatrix}
	S^{0,0} &          &           &            &          &        \\ 
	& \ddots   &           &            &          &        \\
	&          & S^{0,0}   &            &          &        \\
	&          &           &  S^{0,0}   &          &        \\
	&          &           &            &  \ddots  &        \\ 
	&          &           &            &          &S^{0,0} \\
	&          &           &  S^{0,0}   &          &        \\
	&          &           &            &  \ddots  &        \\ 
	&          &           &            &          &S^{0,0} \\
	&          &           &  S^{0,0}   &          &        \\
	&          &           &            &  \ddots  &        \\ 
	&          &           &            &          &S^{0,0}               
	\end{pmatrix},
	~~\bm{b}^{n+1,1}:=-\begin{pmatrix}
	\bm{0}_c\\
	\bm{0}_c\\
	(\frac{\bm{1}}{\bm 2})_{c}\\
	\bm{1}_{c}
	\end{pmatrix}.
	\end{equation*}
	By the inductive hypothesis, we have 
	\begin{equation*}
	\bm{f}^{n+1,1}:=\sigma\circ A_{K^{n+1,1},\bm{b}^{n+1,1}}\circ \bm{f}^{n}\in \mathcal{C}^{4c,2n+2}_{2k+1}(\mathbb{R}^{c\times d\times d},\mathbb{R}^{4c\times d\times d}).
	\end{equation*}
	Moreover, for any $X\in[0,1]^{c\times d\times d}$,
	\begin{equation*} 
	\bm{f}^{n+1,1}(X)=
	\begin{pmatrix}
	\sq_n(X)\\
	\sigma(g_n(X))\\
	\sigma\big(g_n(X)-(\frac{\bm{1}}{\bm 2})_{c}\bm{1}_{d\times d}\big)\\
	\sigma\big(g_n(X)-\bm{1}_{c}\bm{1}_{d\times d}\big)
	\end{pmatrix}.
	\end{equation*}	
	Next, we define the kernel $K^{n+1,2}\in\mathbb{R}^{2c\times 4c\times (2k+1)\times (2k+1)}$ as 
	\begin{equation*}
	\setcounter{MaxMatrixCols}{20}
	\scalebox{0.82}{$
	K^{n+1,2}
	:=\begin{pmatrix}
	S^{0,0} &      && -\frac{S^{0,0}}{2^{2n+1}}&      && \frac{S^{0,0}}{2^{2n}} &        &&-\frac{S^{0,0}}{2^{2n+1}}& & \\
	&\ddots&&               &\ddots&&         &\ddots  &&              &\ddots & \\
	&       &S^{0,0} &  &  &-\frac{S^{0,0}}{2^{2n+1}}& &  & \frac{S^{0,0}}{2^{2n}} &  & & -\frac{S^{0,0}}{2^{2n+1}}\\
	&       && 2S^{0,0}  &  &&-4S^{0,0}  &  &  & 2S^{0,0} & & \\
	&       &&        &\ddots  &&  &\ddots  &  &  &\ddots & \\
	&       &&        &  &2S^{0,0}   &  &  &-4S^{0,0}   &  && 2S^{0,0} 
	\end{pmatrix}
	$}
	\end{equation*}		
	and the bias $\bm{b}^{n+1,2}\in\mathbb{R}^{2c}$ as $\bm{b}^{n+1,2}:=\bm{0}_{2c}$. For any $X\in[0,1]^{c\times d\times d}$, we have
	\begin{equation*} 
	\bm{f}^{n+1}(X):=
	\begin{pmatrix}
	\sq_{n+1}(X)\\
	g_{n+1}(X)
	\end{pmatrix}=\sigma\circ A_{K^{n+1,2},\bm{b}^{n+1,2}}\circ\bm{f}^{n+1,1}(X).
	\end{equation*}	
	Then, in light of $\bm{f}^{n+1,1}\in \mathcal{C}^{4c,2n+2}_{2k+1}(\mathbb{R}^{c\times d\times d},\mathbb{R}^{4c\times d\times d})$, we conclude that
	\begin{equation*}
	\bm{f}^{n+1}\in\mathcal{C}^{4c,2(n+1)+1}_{2k+1}(\mathbb{R}^{c\times d\times d},\mathbb{R}^{2c\times d\times d}).
	\end{equation*}
	By the principle of induction, the assertion holds for all $n\in\mathbb{N}$.
	
	Finally, to complete the proof, we need to show $\sq_{n}\in \mathcal{C}^{4c,2(n+1)}_{2k+1}(\mathbb{R}^{c\times d\times d},\mathbb{R}^{c\times d\times d})$. We accomplish this by projecting $\bm{f}^n(X)$ onto the first 
	$c$ channels. Specifically, we define the kernel $K^{n}\in\mathbb{R}^{c\times 2c \times(2k+1)\times(2k+1)}$ and bias $\bm{b}^{n}\in\mathbb{R}^{c}$ as follows
	\begin{equation*}
	K^{n}:=
	\begin{pmatrix}
	S^{0,0}    &          &          &    \bm{0}_{(2k+1)\times(2k+1)}  &\cdots& \bm{0}_{(2k+1)\times(2k+1)}     \\
	& \ddots   &          &    \vdots &\vdots& \vdots\\
	&          &  S^{0,0} &    \bm{0}_{(2k+1)\times(2k+1)}  &\cdots& \bm{0}_{(2k+1)\times(2k+1)}              
	\end{pmatrix},
	~~\bm{b}^{n}:=\bm{0}_{c}.
	\end{equation*}
	Using this kernel and bias, we obtain for any $X\in[0,1]^{c\times d\times d}$,
	\begin{equation*}
	\sq_{n}(X)=\sigma\circ A_{K^{n},\bm{b}^{n}}\circ\bm{f}^{n}(X),
	\end{equation*}	
	from which it follows that $\sq_{n}\in \mathcal{C}^{4c,2(n+1)}_{2k+1}(\mathbb{R}^{c\times d\times d},\mathbb{R}^{c\times d\times d})$.
\end{proof}

Using the function $\sq_{n}:[0,1]\rightarrow[0,1]$, we construct a mapping $\prd_n:[0,1]^2\rightarrow\mathbb{R}$ to approximate product of numbers from $[0,1]$. Specifically,
for $n\in\mathbb{N}$, we define
\begin{equation*}
\prd_n(x,y):=2\bigg(\sq_{n}\Big(\frac{x+y}{2}\Big)-\sq_{n}\Big(\frac{x}{2}\Big)-\sq_{n}\Big(\frac{y}{2}\Big)\bigg),\quad x,y\in[0,1].
\end{equation*}
As demonstrated in He et al. \cite{he2022relu}, the function
$\prd_n(x,y)$ is exactly the piecewise linear interpolant of the function $xy$ over $[0,1]^2$ on a uniform mesh of size $2^{-(n+1)}$. From this perspective, the properties stated in Lemma \ref{lem:prodappr4scalar} follow naturally from standard results on piecewise linear interpolation, in a manner analogous to the earlier identification of $\sq_n(x)$ as a piecewise linear interpolant of $x^2$, and are closely connected to hierarchical bases in finite element spaces \cite{he2022relu}. For the sake of completeness, however, we provide a self-contained proof of the lemma.

\begin{lemma}\label{lem:prodappr4scalar}
	For any $n\in\mathbb{N}$, the mapping $\prd_n:[0,1]^2\rightarrow\mathbb{R}$ satisfies
	\begin{itemize}
		\item [(a)] for any $x,y\in[0,1]$, $\prd_n(x,y)\in [0,1]$;
		\item [(b)] if $x=0$ or $y=0$, then $\prd_n(x,y)=0$;
		\item [(c)] if $x=1$ (respectively, $y=1$), then $\prd_n(x,y)=y$ (respectively, $\prd_n(x,y)=x$);
		\item [(d)] for any $x,y\in[0,1]$, $|\prd_n(x,y)-xy|\leq 3\cdot2^{-2n-1}$.
	\end{itemize}
\end{lemma}
\begin{proof}
	To prove part (a), note that for given $x,y\in[0,1]$, there exist positive integers $i,j\in 1:(2^n-1)$, such that
	\begin{equation*}
	x\in[i2^{-n},(i+1)2^{-n}],~~~y\in[j2^{-n},(j+1)2^{-n}].
	\end{equation*}
	Consequently,
	\begin{equation*}
	\frac{x}{2}\in\bigg[\frac{i}{2^{n+1}},\frac{i+1}{2^{n+1}}\bigg],~~~\frac{y}{2}\in\bigg[\frac{j}{2^{n+1}},\frac{j+1}{2^{n+1}}\bigg],~~~\frac{x+y}{2}\in\bigg[\frac{i+j}{2^{n+1}},\frac{i+j+2}{2^{n+1}}\bigg].
	\end{equation*}
	Recall that $\sq_{n}(x)$ is the piecewise linear interpolation of $f(x)=x^2$ with $2^n+1$ uniformly distributed breakpoints $\frac{0}{2^n},\frac{1}{2^n},\ldots,\frac{2^n}{2^n}$:
	\begin{equation*}
	\sq_{n}\Big(\frac{l}{2^n}\Big)=\Big(\frac{l}{2^n}\Big)^2, \quad l\in 0:2^n.
	\end{equation*}
	It follows that
	\begin{eqnarray*}
		\sq_{n}(x)=2^{-n}((2i+1)x-i(i+1)2^{-n}).
	\end{eqnarray*}
	
	The remainder of the proof for part (a) is divided into three cases:
	(1) both $i$ and $j$ are even; (2) both $i$ and $j$ are odd; (3) one of $i$ and $j$ is even and the other is odd. For the sake of brevity, we will only prove the first case in detail, as the other two cases can be addressed using a similar approach.
	
	Suppose that both $i$ and $j$ are even. Then, we have
	\begin{equation*}
	\frac{x}{2}\in\bigg[\frac{i}{2^{n+1}},\frac{i+2}{2^{n+1}}\bigg],~~~\frac{y}{2}\in\bigg[\frac{j}{2^{n+1}},\frac{j+2}{2^{n+1}}\bigg],~~~\frac{x+y}{2}\in\bigg[\frac{i+j}{2^{n+1}},\frac{i+j+2}{2^{n+1}}\bigg],
	\end{equation*}
	and consequently,
	\begin{align*}
		\sq_{n}\Big(\frac{x}{2}\Big)&=2^{-n}\Big((i+1)\frac{x}{2}-\frac{i}{2}\Big(\frac{i}{2}+1\Big)2^{-n}\Big),\\
		\sq_{n}\Big(\frac{y}{2}\Big)&=2^{-n}\Big((j+1)\frac{y}{2}-\frac{j}{2}\Big(\frac{j}{2}+1\Big)2^{-n}\Big),\\
		\sq_{n}\Big(\frac{x+y}{2}\Big)&=2^{-n}\Big((i+j+1)\frac{x+y}{2}-\frac{i+j}{2}\Big(\frac{i+j}{2}+1\Big)2^{-n}\Big).
	\end{align*}
	Therefore,
	\begin{eqnarray*}
		\prd_n(x,y)=2\Big(\sq_{n}\Big(\frac{x+y}{2}\Big)-\sq_{n}\Big(\frac{x}{2}\Big)-\sq_{n}\Big(\frac{y}{2}\Big)\Big)=2\cdot2^{-n}\Big(j\cdot\frac{x}{2}+i\cdot\frac{y}{2}-\frac{ij}{2}\cdot2^{-n}\Big).
	\end{eqnarray*}
	Noting that $x\in[i2^{-n},(i+1)2^{-n}]$ and $y\in[j2^{-n},(j+1)2^{-n}]$, we obtain
	\begin{eqnarray*}
		&&\prd_n(x,y)\geq 2\cdot2^{-n}\Big(j\cdot\frac{i2^{-n}}{2}+i\cdot\frac{j2^{-n}}{2}-\frac{ij}{2}\cdot2^{-n}\Big)=0,\\
		&&\prd_n(x,y)\leq 2\cdot2^{-n}\Big(j\cdot\frac{(i+1)2^{-n}}{2}+i\cdot\frac{(j+1)2^{-n}}{2}-\frac{ij}{2}\cdot2^{-n}\Big)\leq 1-2^{-2n}\leq 1,
	\end{eqnarray*}
	i.e., $\prd_n(x,y)\in [0,1]$. This proves part (a).
	
	Part (b) follows directly from the definition of $\prd_n$.
	
	To prove part (c), we start by noting that for $m\geq 2$, $g_m$ satisfies $g_m(\frac{y+1}{2})=g_m(\frac{y}{2})$ for any $y\in[0,1]$.
	Hence, we can compute $\prd_n(1,y)$ as follows:
	\begin{align*}
		\prd_n(1,y)&=2\bigg(\sq_{n}\Big(\frac{1+y}{2}\Big)-\sq_{n}\Big(\frac{1}{2}\Big)-\sq_{n}\Big(\frac{y}{2}\Big)\bigg)\\
		&=2\bigg(\frac{1+y}{2}-\sum_{m=1}^n\frac{g_m(\frac{1+y}{2})}{4^m}-\frac{1}{4}-\frac{y}{2}+\sum_{m=1}^n\frac{g_m(\frac{y}{2})}{4^m}\bigg)\\
		&=2\bigg(\frac{1}{4}-\frac{g_1(\frac{1+y}{2})}{4}+\frac{g_1(\frac{y}{2})}{4}\bigg).
	\end{align*}
	By the definition of $g_1$, we find that for any $y\in[0,1]$,
	\begin{equation*}
	g_1\Big(\frac{y}{2}\Big)-g_1\Big(\frac{y+1}{2}\Big)=2y-1.
	\end{equation*}
	Substituting this into our expression, we obtain:
	\begin{eqnarray*}
		\prd_n(1,y)=y,
	\end{eqnarray*}
	from which part (c) follows.
	
	We prove part (d). It follows from the identity $xy=2\Big(\big(\frac{x+y}{2}\big)^2-\big(\frac{x}{2}\big)^2-\big(\frac{y}{2}\big)^2\Big)$ that for any $x,y\in[0,1]$,
	\begin{align*}
		\Big|\prd_n(x,y)-xy\Big|
		\leq&2 \Big|\sq_{n}\Big(\frac{x+y}{2}\Big)-\Big(\frac{x+y}{2}\Big)^2\Big|+2 \Big|\sq_{n}\Big(\frac{x}{2}\Big)-\Big(\frac{x}{2}\Big)^2\Big|\\
		&+2 \Big|\sq_{n}\Big(\frac{y}{2}\Big)-\Big(\frac{y}{2}\Big)^2\Big|\\
		\leq& 3\cdot 2^{-2n-1}, 	
	\end{align*}
	as claimed.
\end{proof}

When applied to tensors, the map $\prd_n$ is interpreted as performing component-wise operations. By employing $\prd_n$ in this manner, we can prove Proposition \ref{pro:app4prod4tensor} as follows.
\begin{proof}[Proof of Proposition \ref{pro:app4prod4tensor}]
	For simplicity, we will only consider the case $d=2^p$ for some $p\in\mathbb{N}$. Given any $X\in [0,1]^{d\times d}$, we define $X^{0}:=X$. According to Lemma \ref{lem:prodappr4scalar} (a), we can recursively construct a sequence of tensors $X^q\in [0,1]^{d\times 2^{p-q}}$ for $q\in 1:p$ as follows
	\begin{equation*}
	[X^q]_{:,j}:=\prd_n([X^{q-1}]_{:,2j-1},[X^{q-1}]_{:,2j}),\quad j\in 1:2^{p-q}.
	\end{equation*}
	It follows from Lemma \ref{lem:prodappr4scalar} (d) that
	\begin{equation*}
	\Big|[X^1]_{:,j}-[X]_{:,2j-1}\odot[X]_{:,2j}\Big|\leq 3\cdot 2^{-2n-1},\quad j\in 1:2^{p-1}.
	\end{equation*}
	where both the product $\odot$ and the inequality $\leq$ are understood component-wise. By induction, it is straightforward to derive the following inequality:
	\begin{equation}\label{eq:rowproerr}
	\Big|X^p-\odot_{j=1}^{2^p}[X]_{:,j}\Big|\leq 3\cdot 2^{-2n-1}(2^p-1).
	\end{equation}
	
	We first present the claim: There exists a mapping $\widetilde{\Pi}_n^{\text{c}} \in\mathcal{C}^{12,L^{\text{c}}}_{2k+1}(\mathbb{R}^{d\times d},\mathbb{R}^{d\times d})$ with $L^{\text{c}}=(2n+3)p+(d-1)$ such that
	\begin{equation*}
	[\widetilde{\Pi}_n^{\text{c}}(X)]_{:,2^p}=X^p.
	\end{equation*}
	Let $S^{s,t}\in\mathbb{R}^{(2k+1)\times(2k+1)}$, with $s,t\in -k:k$, be the basic blocks defined in Section \ref{sec:proof4thm1}. We consider the following kernels:
	\begin{align*}
	K^{0}:=\frac{1}{2}
	\begin{pmatrix}
	S^{0,-1} + S^{0,0} \\
	S^{0,-1} \\
	S^{0,0}                           
	\end{pmatrix},~~K^{1}:=
	\begin{pmatrix}
	S^{0,-1} \\
	S^{0,0}                         
	\end{pmatrix},~~K^{2}:=
	\begin{pmatrix}
	S^{0,-1} &  \\
	& S^{0,0}                      
	\end{pmatrix},\\	
	K^{3}:=\frac{1}{2}
	\begin{pmatrix}
	S^{0,-1} &   S^{0,0}     \\
	S^{0,-1} &   \bm{0}   \\
	\bm{0} &   S^{0,0}                    
	\end{pmatrix},~~
	K^{4}:=2
	\begin{pmatrix}
	S^{0,0} & -S^{0,0} & -S^{0,0}                           
	\end{pmatrix}.
	\end{align*}	
	We set $\Lambda_1:=\sigma\circ A_{K^{4}}\circ\sq_n\circ\sigma\circ A_{K^{0}}$. By Lemma \ref{lem:squappr4tensor}, $\Lambda_1 \in\mathcal{C}^{12,2(n+2)}_{2k+1}(\mathbb{R}^{d\times d},\mathbb{R}^{d\times d})$. Let $Y^{1}:=\Lambda_1(X)$. Direct computation shows that
	\begin{equation*}
	[Y^{1}]_{:,2j}=[X^1]_{j},\quad j\in 1:2^{p-1}.
	\end{equation*}
	For $q\in 2:p$, we put
	\begin{equation*}
	\Lambda_q:=\sigma\circ A_{K^{4}}\circ\sq_n\circ\sigma\circ A_{K^{3}}\circ\underbrace{(\sigma\circ A_{K^{2}})\circ\cdots\circ(\sigma\circ A_{K^{2}})}_{2^{q-1}-2}\circ \sigma\circ A_{K^{1}}.
	\end{equation*}   
	By applying Lemma \ref{lem:squappr4tensor} again, $\Lambda_q \in\mathcal{C}^{12,L_q^{\text{c}}}_{2k+1}(\mathbb{R}^{d\times d},\mathbb{R}^{d\times d})$ with $L_q^{\text{c}}=2(n+1)+2^{q-1}+1$. We define recursively the tensor sequence $Y^{1},Y^{2},\ldots,Y^{p}\in [0,1]^{d\times d}$ by
	\begin{equation*}
	Y^{q}:=\Lambda_q(Y^{q-1}),\quad q\in 2:p.
	\end{equation*} 
	With this construction, we obtain the following relationship:
	\begin{equation*}
	[Y^{q}]_{:,j2^q}=\prd_n([Y^{q-1}]_{:,(2j-1)2^{q-1}},[Y^{q-1}]_{:,(2j)2^{q-1}}), \quad j\in 1:2^{p-q}.
	\end{equation*}
	Additionally, from the equality $[Y^{1}]_{:,2j}=[X^1]_{:,j}$, we find that
	\begin{equation*}
	[Y^{2}]_{:,j2^2}=\prd_n([X^{1}]_{:,(2j-1)},[X^{1}]_{:,(2j)})=[X^2]_{:,j}, \quad j\in 1:2^{p-2}.
	\end{equation*}
	Repeating this process, we eventually arrive at
	\begin{equation*}
	[Y^{p}]_{:,2^p}=\prd_n([X^{p-1}]_{:,1},[X^{p-1}]_{:,2})=X^p.
	\end{equation*}
	Define $\widetilde{\Pi}_n^{\text{c}}:=\Lambda_p\circ\Lambda_{p-1}\circ\cdots\circ\Lambda_1$. Then, by what was shown, we have
	\begin{equation*}
	[\widetilde{\Pi}_n^{\text{c}}(X)]_{:,2^p}=X^p,
	\end{equation*}
	and consequently,
	\begin{equation*}
	\Big|\widetilde{\Pi}_n^{\text{c}}(X)]_{:,2^p}-\odot_{j=1}^{2^p}[X]_{:,j}\Big|\leq 3\cdot 2^{-2n-1}(2^p-1).
	\end{equation*}
	Furthermore, it can be verified that $\widetilde{\Pi}_n^{\text{c}} \in\mathcal{C}^{12,L^{\text{c}}}_{2k+1}(\mathbb{R}^{d\times d},\mathbb{R}^{d\times d})$ with $L^{\text{c}}=(2n+3)p+(d-1)$. This completes the proof of the claim.
	
	Next, starting with the tensor $Y^p\in[0,1]^{d\times d}$, we define $Z^0:=Y^p$ and construct recursively the sequence of tensors $Z^q\in [0,1]^{2^{p-q}\times d}$ for $q\in 1:p$ as follows:
	\begin{equation*}
	[Z^q]_{i,:}:=\prd_n([Z^{q-1}]_{2i-1,:},[Z^{q-1}]_{2i,:}),\quad\text{for}~i\in 1:2^{p-q}.
	\end{equation*}
	From the inequality (\ref{eq:rowproerr}), we obtain
	\begin{align*}
		\bigg|[Z^1]_{i,d}-\prod_{j=1}^{2^p}([X]_{2i-1,j}\cdot[X]_{2i,j})\bigg|\leq&\bigg|\prd_n([Z^{0}]_{2i-1,d},[Z^{0}]_{2i,d})-[Z^{0}]_{2i-1,d}\cdot[Z^{0}]_{2i,d}\bigg|\\
		&+\bigg|[Z^{0}]_{2i-1,d}\cdot[Z^{0}]_{2i,d}-\prod_{j=1}^{2^p}[X]_{2i-1,j}[Z^{0}]_{2i,d}\bigg|\\
		&+\bigg|\prod_{j=1}^{2^p}[X]_{2i-1,j}[Z^{0}]_{2i,d}-\prod_{j=1}^{2^p}[X]_{2i-1,j}\prod_{j=1}^{2^p}[X]_{2i,j}\bigg|\\
		\leq&3\cdot 2^{-2n-1}(2^{p+1}-1).
	\end{align*}
	By recursively applying the process, we can derive
	\begin{eqnarray*}
		\bigg|[Z^p]_{1,d}-\prod_{i,j=1}^d[X]_{i,j}\bigg|\leq 3\cdot 2^{-2n-1}(2^{2p}-1).
	\end{eqnarray*}
	
	Similarly, for $Z^q$, we assert that there exists a mapping $\widetilde{\Pi}_n^{\text{r}} \in\mathcal{C}^{12,L^{\text{r}}}_{2k+1}(\mathbb{R}^{d\times d},\mathbb{R}^{d\times d})$ with $L^{\text{r}}=(2n+3)p+(d-1)$ such that
	\begin{equation*}
	[\widetilde{\Pi}_n^{\text{r}}(Y^p)]_{2^p,:}=Z^p.
	\end{equation*}   
	The proof of this assertion closely follows the methodology used for the previous claim. The primary difference lies in a minor modification where we replace the basic block $S^{0,-1}$ with $S^{-1,0}$. This adjustment accounts for the different orientation of the tensor operations, transitioning from column-wise operations to row-wise operations. All other aspects of the proof, including the definition of the kernels, the application of the squaring operation, and the recursive construction of the tensor sequence, remain the same.
	
	Finally, let $\widetilde{\Pi}_n:=\widetilde{\Pi}_n^{\text{r}}\circ\widetilde{\Pi}_n^{\text{c}}$. Then, $\widetilde{\Pi}_n\in\mathcal{C}^{12,L}_{2k+1}(\mathbb{R}^{d\times d},\mathbb{R}^{d\times d})$ with $L:=L^{\text{c}}+L^{\text{r}}=2(2n+3)p+2(d-1)$. Moreover, we have 
	\begin{equation*}
	\bigg|[\widetilde{\Pi}_n(X)]_{d,d}-\prod_{i,j=1}^d[X]_{i,j}\bigg|=\bigg|[Z^p]_{1,d}-\prod_{i,j=1}^d[X]_{i,j}\bigg|\leq 3\cdot 2^{-2n-1}(2^{2p}-1),
	\end{equation*}
	which completes the proof of Proposition \ref{pro:app4prod4tensor}. 
\end{proof}

\section{Proof of Proposition \ref{pro:app4prod4basis}}\label{app:proof4prod4basis}
\begin{lemma}\label{lem:takcomp}
	Let $k,d\in\mathbb{N}$. For any $m,n\in 1:d$, define a mapping $\Delta_{m,n}:\mathbb{R}^{d\times d}\rightarrow\mathbb{R}^{d\times d}$ by
	\begin{equation*}
	[\Delta_{m,n}(X)]_{m^\prime,n^\prime}:=	
	\begin{cases}
	[X]_{m,n}, & \text{if}~m^\prime=m~\text{and}~n^\prime=n, \\
	0, & \text{otherwise}.
	\end{cases}
	\end{equation*}
	There exists a finite sequence of
	kernels $K^1,K^2,\ldots,K^r\in \mathbb{R}^{(2k+1)\times (2k+1)}$, dependent on $m,n$, with $r\leq \frac{5}{2}d-1$, such that for any $X\in [0,1]^{d\times d}$,
	\begin{equation*}
	\Delta_{m,n}(X)=K^r\ast \cdots \ast K^2\ast K^1\ast X,
	\end{equation*}
	where $\ast$ denotes the convolution operation defined in Subsection \ref{subsec:2drelucnn}.
\end{lemma}
\begin{proof}	
	We split the proof into four cases based on different ranges of $m$ and $n$: (1) $m,n\in 1:\lceil \frac{d}{2} \rceil$; (2) $m\in 1:\lceil \frac{d}{2} \rceil~\text{and}~n\in (\lceil \frac{d}{2} \rceil+1):d$; (3) $m\in (\lceil \frac{d}{2} \rceil+1):d~\text{and}~n\in 1:\lceil \frac{d}{2} \rceil$; (4) $m,n\in (\lceil \frac{d}{2} \rceil+1):d$. We focus on proving the lemma for the first two cases only. The proofs for Cases 3 and 4 can be obtained from the results of Cases 2 and 1, respectively.
	
	We prove the lemma for Case 1: $m,n\in 1:\lceil \frac{d}{2} \rceil$. If $m\leq n$, then for any $X\in [0,1]^{d\times d}$, $\Delta_{m,n}(X)$ can be expressed as a series of convolutions using specific kernels:
	\begin{align*}
	\Delta_{m,n}(X)=&\underbrace{S^{1,0}\ast \cdots \ast S^{1,0}}_{n-m} \ast \underbrace{S^{1,1}\ast \cdots \ast S^{1,1}}_{d-n}\ast \underbrace{S^{-1,-1}\ast \cdots \ast S^{-1,-1}}_{d-1}\ast \underbrace{S^{0,1}\ast \cdots \ast S^{0,1}}_{n-m}\\
	&\ast \underbrace{S^{1,1}\ast \cdots \ast S^{1,1}}_{m-1} \ast X,
	\end{align*}
	where $S^{s,t}\in\mathbb{R}^{(2k+1)\times(2k+1)}$, with $s,t\in -k:k$, are the basic blocks defined in Section \ref{sec:proof4thm1}.
	If $n<m$, a similar expression holds
	\begin{align*}
	\Delta_{m,n}(X)=&\underbrace{S^{0,1}\ast \cdots \ast S^{0,1}}_{m-n} \ast \underbrace{S^{1,1}\ast \cdots \ast S^{1,1}}_{d-m}\ast \underbrace{S^{-1,-1}\ast \cdots \ast S^{-1,-1}}_{d-1}\ast \underbrace{S^{1,0}\ast \cdots \ast S^{1,0}}_{m-n}\\
	&\ast \underbrace{S^{1,1}\ast \cdots \ast S^{1,1}}_{n-1} \ast X.
	\end{align*}
	Hence, the desired result follows for Case 1.
	
	We prove the lemma for Case 2: $m\in 1:\lceil \frac{d}{2} \rceil~\text{and}~n\in (\lceil \frac{d}{2} \rceil+1):d$. If $m+n\leq d+1$, then for any $X\in [0,1]^{d\times d}$, $\Delta_{m,n}(X)$ can be expressed as
	\begin{align*}
	\Delta_{m,n}(X)=&\underbrace{S^{1,0}\ast \cdots \ast S^{1,0}}_{d+1-m-n} \ast \underbrace{S^{1,-1}\ast \cdots \ast S^{1,-1}}_{n-1}\ast \underbrace{S^{-1,1}\ast \cdots \ast S^{-1,1}}_{d-1}\\
	&\ast \underbrace{S^{0,-1}\ast \cdots \ast S^{0,-1}}_{d+1-m-n}\ast \underbrace{S^{1,-1}\ast \cdots \ast S^{1,-1}}_{m-1} \ast X.
	\end{align*}
	If $m+n>d+1$, a similar expression holds
	\begin{align*}
	\Delta_{m,n}(X)=&\underbrace{S^{0,-1}\ast \cdots \ast S^{0,-1}}_{m+n-d-1} \ast \underbrace{S^{1,-1}\ast \cdots \ast S^{1,-1}}_{d-m}\ast \underbrace{S^{-1,1}\ast \cdots \ast S^{-1,1}}_{d-1}\\
	&\ast \underbrace{S^{1,0}\ast \cdots \ast S^{1,0}}_{m+n-d-1}\ast \underbrace{S^{1,-1}\ast \cdots \ast S^{1,-1}}_{d-n} \ast X.
	\end{align*}
	Thus, the claimed result follows for Case 2.	
\end{proof}

Let $h_l$, $x_{l,i}$, $\phi_{l, i}$, and $I_{\bm{l}}$ be defined as in Subsection \ref{subsec:korov}, where $D$ is replaced by $d^2$. Using these notations, we present the following lemma.
\begin{lemma}\label{lem:genbas}
	Let $k,d\in\mathbb{N}$. For any $\bm{l}\in\mathbb{N}^{d^2}$ and $\bm{i}\in I_{\bm{l}}$, define a mapping $\Phi_{\bm{l},\bm{i}}:[0,1]^{d\times d}\rightarrow[0,1]^{d\times d}$ by
	\begin{equation*}
	[\Phi_{\bm{l},\bm{i}}(X)]_{m,n}=\phi_{l_j, i_j}\big([X]_{m,n}\big), \quad X\in[0,1]^{d\times d}, 
	\end{equation*}
	where $j=j(m,n):=(m-1)d+n$ for $m,n\in 1:d$. Then, we have 
	\begin{equation*}
	\Phi_{\bm{l},\bm{i}}\in \mathcal{C}^{W,L}_{2k+1}(\mathbb{R}^{d\times d},\mathbb{R}^{d\times d}),
	\end{equation*}
	with $W=2d^2$ and $L=\lfloor\frac{5}{2}d\rfloor+3$.
\end{lemma}
\begin{proof}	
	First, set
	\begin{equation*}
	K^1:=\begin{pmatrix}
	\frac{1}{h_{l_1}}S^{0,0}\\
	-\frac{1}{h_{l_1}}S^{0,0}\\
	\vdots\\
	\frac{1}{h_{l_{d^2}}}S^{0,0}\\
	-\frac{1}{h_{l_{d^2}}}S^{0,0}
	\end{pmatrix}\in\mathbb{R}^{2d^2\times 1\times(2k+1)\times(2k+1)},~~
	\bm{b}^1:=\begin{pmatrix}
	-\frac{x_{l_1,i_1}}{h_{l_1}}\\
	\frac{x_{l_1,i_1}}{h_{l_1}}\\
	\vdots\\
	-\frac{x_{l_{d^2},i_{d^2}}}{h_{l_{d^2}}}\\
	\frac{x_{l_{d^2},i_{d^2}}}{h_{l_{d^2}}}
	\end{pmatrix}\in\mathbb{R}^{2d^2}.
	\end{equation*}
	We obtain for each $j\in 1:d^2$, 		
	\begin{align*}
	\big[\sigma\circ A_{K^1,\bm{b}^1}(X)\big]_{2j-1,m(j),n(j)} 
	&= \sigma\bigg(\frac{[X]_{m(j),n(j)}-x_{l_j,i_j}}{h_{l_j}}\bigg), \\
	\big[\sigma\circ A_{K^1,\bm{b}^1}(X)\big]_{2j,m(j),n(j)} 
	&= \sigma\bigg(\frac{x_{l_j,i_j}-[X]_{m(j),n(j)}}{h_{l_j}}\bigg),
	\end{align*}			
	where $m(j)=\left\lfloor\frac{j-1}{d}\right\rfloor+1$ and $n(j)=j-d\left\lfloor\frac{j-1}{d}\right\rfloor$.
	Next, define the kernel $K^2\in\mathbb{R}^{d^2\times 2d^2 \times(2k+1)\times(2k+1)}$ and bias $\bm{b}^2\in\mathbb{R}^{d^2}$ as
	\begin{equation*}
	K^2:=
	\begin{pmatrix}
	-S^{0,0} & -S^{0,0} &          &        &          & \\
	& -S^{0,0} & -S^{0,0} &        &          & \\
	&          & \ddots   & \ddots &          & \\
	&          &          &        & -S^{0,0} & -S^{0,0}         
	\end{pmatrix}
	,~~\bm{b}^2:=\bm{1}_{d^2}.
	\end{equation*}
	Noting that 
	\begin{equation*}
	\sigma\bigg(1-\sigma\bigg(\frac{[X]_{m(j),n(j)}-x_{l_j,i_j}}{h_{l_j}}\bigg)-\sigma\bigg(\frac{x_{l_j,i_j}-[X]_{m(j),n(j)}}{h_{l_j}}\bigg)\bigg)=\phi_{l_j, i_j}\big([X]_{m(j),n(j)}\big),
	\end{equation*}        
	we have for each $j\in 1:d^2$,
	\begin{equation*}
	\big[\sigma\circ A_{K^2,\bm{b}^2}\circ\sigma\circ A_{K^1,\bm{b}^1}(X)\big]_{j,m(j),n(j)}=\phi_{l_j, i_j}\big([X]_{m(j),n(j)}\big).
	\end{equation*}		
	By Lemma \ref{lem:takcomp}, there exists a network $\Delta_{m(j),n(j)}\in\mathcal{C}^{1,\lfloor\frac{5}{2}d\rfloor}_{2k+1}(\mathbb{R}^{d\times d},\mathbb{R}^{d\times d})$ such that
	\begin{align*}
	&\begin{pNiceMatrix}[first-row,first-col]
	&   &    & \underset{\downarrow}{n(j)} &  &  \\
	& 0 & \cdots & 0 & \cdots & 0 \\
	& \vdots & \ddots & \vdots & \ddots & \vdots \\
	m(j){\rightarrow} & 0 & \cdots & \phi_{l_j, i_j}\big([X]_{m(j),n(j)}\big) & \cdots & 0 \\
	& \vdots & \ddots & \vdots & \ddots & \vdots \\
	& 0 & \cdots & 0 & \cdots & 0 \\
	\end{pNiceMatrix}\\
	&\hspace{15em}=\Delta_{m(j),n(j)}\big([\sigma\circ A_{K^2,\bm{b}^2}\circ\sigma\circ A_{K^1,\bm{b}^1}(X)]_{j,:,:}\big).
	\end{align*}	
	Let $\Delta$ represent the concatenation of the networks $\Delta_{m(j),n(j)}$ for $1\leq j\leq d^2$
	\begin{equation*}
	\Delta=\Delta_{m(1),n(1)}\oplus\Delta_{m(2),n(2)}\oplus\cdots\oplus\Delta_{m(d^2),n(d^2)}.
	\end{equation*}
	Then, by Lemma \ref{lem:con_net}, $\Delta\in\mathcal{C}^{d^2,\lfloor\frac{5}{2}d\rfloor}_{2k+1}(\mathbb{R}^{d^2\times d\times d},\mathbb{R}^{d^2\times d\times d})$ and we observe that
	\begin{equation*}
	\begin{pNiceMatrix}[first-row,first-col]
	&   &    & \underset{\downarrow}{n(j)} &  &  \\
	& 0 & \cdots & 0 & \cdots & 0 \\
	& \vdots & \ddots & \vdots & \ddots & \vdots \\
	m(j){\rightarrow} & 0 & \cdots & \phi_{l_j, i_j}\big([X]_{m(j),n(j)}\big) & \cdots & 0 \\
	& \vdots & \ddots & \vdots & \ddots & \vdots \\
	& 0 & \cdots & 0 & \cdots & 0 \\
	\end{pNiceMatrix}=\big[\Delta\circ\sigma\circ A_{K^2,\bm{b}^2}\circ\sigma\circ A_{K^1,\bm{b}^1}(X)\big]_{j,:,:}.
	\end{equation*}
	Finally, by taking 
	\begin{equation*}
	K^3=\begin{pmatrix}
	S^{0,0},S^{0,0},\ldots,S^{0,0}
	\end{pmatrix}
	\in\mathbb{R}^{1\times d^2\times(2k+1)\times(2k+1)}
	,~~b^3=0\in\mathbb{R},
	\end{equation*}
	we have for $m,n\in 1:d$,
	\begin{equation*}
	[\sigma\circ A_{K^3,b^3}\circ \Delta\circ\sigma\circ A_{K^2,\bm{b}^2}\circ\sigma\circ A_{K^1,\bm{b}^1}(X)]_{m,n}=\phi_{l_j, i_j}\big([X]_{m,n}\big),
	\end{equation*}
	where $j=j(m,n):=(m-1)d+n$. The desired conclusion then follows by letting $\Phi_{\bm{l},\bm{i}}=\sigma\circ A_{K^3,b^3}\circ\Delta\circ\sigma\circ A_{K^2,\bm{b}^2}\circ\sigma\circ A_{K^1,\bm{b}^1}$.
\end{proof}

With Lemmas \ref{lem:takcomp} and \ref{lem:genbas} established, we can proceed to prove Proposition \ref{pro:app4prod4basis}.
\begin{proof}[Proof of Proposition \ref{pro:app4prod4basis}]
	For any $n\in\mathbb{N}$, let $\Xi_n$ stand for the set $\{(\bm{l},\bm{i}):|\bm{l}|_1\leq n+d^2-1,\bm{i}\in I_{\bm{l}}\}$. Using the mapping $\widetilde{\Pi}_n$ from Proposition \ref{pro:app4prod4tensor} and the mapping $\Phi_{\bm{l},\bm{i}}$ from Lemma \ref{lem:genbas}, we define, for each $(\bm{l},\bm{i})\in \Xi_n$,
	\begin{equation*}
	\bm{g}_{\bm{l},\bm{i}}:=\widetilde{\Pi}_n\circ\Phi_{\bm{l},\bm{i}}.
	\end{equation*}
	It can be readily verified that  $\bm{g}_{\bm{l},\bm{i}}\in\mathcal{C}^{2d^2,L}_{2k+1}(\mathbb{R}^{d\times d},\mathbb{R}^{d\times d})$, where $L=2(2n+3)\lceil\log_2d\rceil+5d$. By Lemma \ref{lem:prodappr4scalar}, the support of $[\bm{g}_{\bm{l},\bm{i}}(X)]_{d,d}$ is contained within the support of $\phi_{\bm{l},\bm{i}}(\vect(X))$.  
	Moreover, by Proposition \ref{pro:app4prod4tensor}, for $X\in[0,1]^{d\times d}$,
	\begin{equation*}
	\Big|[\bm{g}_{\bm{l},\bm{i}}(X)]_{d,d}-\prod_{i,j=1}^d[\Phi_{\bm {l},\bm {i}}(X)]_{i,j}\Big|=\Big|[\bm{g}_{\bm{l},\bm{i}}(X)]_{d,d}-\phi_{\bm{l},\bm{i}}(\vect(X))\Big|\leq \frac{3}{2}\cdot 2^{-2n}(d^2-1).
	\end{equation*}
	This completes the proof of the proposition.
\end{proof}

\section{Technical Lemmas}\label{app:teclem}
\begin{lemma}\label{lem:boucomb}
	Let $d,n\in\mathbb{N}$. The generating functions for the sequences of combinatorial numbers $\big\{\binom{l+n+d^2-1}{d^2-1}\big\}_{l=0}^{\infty}$ and
	$\big\{\binom{l+d^2-1}{d^2-1}\big\}_{l=0}^{n-1}$ are given by
	\begin{align*}
		\sum_{l=0}^{\infty}\binom{l+n+d^2-1}{d^2-1}x^l&=\frac{1}{1-x}\cdot\sum_{l=0}^{d^2-1}\binom{n+d^2-1}{l}\bigg(\frac{x}{1-x}\bigg)^{d^2-1-l},\\ \sum_{l=0}^{n-1}\binom{l+d^2-1}{d^2-1}x^l&=\bigg(\frac{1}{1-x}\bigg)^{d^2}\cdot\sum_{l=0}^{d^2-1}\bigg[\binom{d^2-1}{l}-\binom{n+d^2-1}{l}x^n\bigg]x^{d^2-1-l}.
	\end{align*}
\end{lemma}
\begin{proof}
	These two equalities can be easily verified.
\end{proof}

\begin{lemma}\label{lem:boundcompl}
	Let $\beta>2$ and $\eta\leq1/3$. Then, for all
	\begin{equation*}
	x\geq(6\beta\log_2\beta)^{\beta}\frac{\big(\log_2\frac{1}{\eta}\big)^{\beta}}{\eta},
	\end{equation*}
    the following inequality holds:
	\begin{equation*}
		\frac{(\log_2x)^{\beta}}{x}\leq \eta.
	\end{equation*}
\end{lemma}
\begin{proof}
	Denote $R=(6\beta\log_2\beta)^{\beta}$ and $T=R\eta^{-1}(\log_2\eta^{-1})^{\beta}$. Given that $\beta>2$ and $\eta\leq1/3$, it follows that
	\begin{equation*}
	T\geq (6\beta\log_23)^{\beta}\geq e^{\beta}.
	\end{equation*}
	Since the function $w(x)=x^{-1}(\log_2x)^{\beta}$ is decreasing on the interval $[e^{\beta},+\infty)$, we have
	\begin{equation*}
	\frac{(\log_2x)^{\beta}}{x}\leq\frac{(\log_2T)^{\beta}}{T}, \qquad x\geq T.
	\end{equation*}
	Thus, the desired inequality follows from
	\begin{align*}
	\frac{(\log_2T)^{\beta}}{T}
	&=\eta\Bigg(\frac{\log_2R+\log_2\frac{1}{\eta}+\beta\log_2(\log_2\frac{1}{\eta})}{R^{\frac{1}{\beta}}\log_2\frac{1}{\eta}}\Bigg)^{\beta}\\
	&\leq\eta\Bigg(\frac{\log_2R}{R^{\frac{1}{\beta}}}+\frac{1}{R^{\frac{1}{\beta}}}+\frac{\beta}{R^{\frac{1}{\beta}}}\cdot\sup_{x>1}\frac{\log_2x}{x}\Bigg)^{\beta}\\
	&\leq\eta\Bigg(\frac{2}{3}+\frac{1}{18}+\frac{1}{6e}\Bigg)^{\beta}\\
	&\leq \eta.
	\end{align*}  
    The proof is completed.
\end{proof}

\end{appendices}

\end{document}